\documentclass[sn-mathphys-num]{sn-jnl}

\usepackage{graphicx}%
\usepackage{multirow}%
\usepackage{amsmath,amssymb,amsfonts}%
\usepackage{amsthm}%
\usepackage{mathrsfs}%
\usepackage[title]{appendix}%
\usepackage{xcolor}%
\usepackage{textcomp}%
\usepackage{manyfoot}%
\usepackage{booktabs}%
\usepackage{algorithm}%
\usepackage{algorithmicx}%
\usepackage{algpseudocode}%
\usepackage{listings}%
\usepackage{array}
\usepackage[caption=false,font=normalsize,labelfont=sf,textfont=sf]{subfig}
\usepackage{stfloats}
\usepackage{url}
\usepackage{verbatim}
\usepackage{multirow}
\usepackage{epsfig}
\usepackage{mathrsfs}

\newtheorem{theorem}{Theorem}
\newtheorem{proposition}{Proposition}

\newtheorem{definition}{Definition}%
\newtheorem{lemma}{Lemma}
\newtheorem{assumption}{Assumption}

\begin{document}

\title{Accelerating Federated Learning by Selecting Beneficial Herd of Local Gradients}


\author[1]{\fnm{Ping} \sur{Luo}}\email{luoping@nudt.edu.cn}
\author[1]{\fnm{Xiaoge} \sur{Deng}}\email{dengxg@nudt.edu.cn}
\author[1]{\fnm{Ziqing} \sur{Wen}}\email{zqwen@nudt.edu.cn}

\author*[1]{\fnm{Tao} \sur{Sun}}\email{nudtsuntao@163.com}
\equalcont{These authors contributed equally to this work.}

\author*[1]{\fnm{Dongsheng} \sur{Li}}\email{dsli@nudt.edu.cn}
\equalcont{These authors contributed equally to this work.}

\affil*[1]{\orgdiv{National Laboratory for Parallel and Distributed Processing}, \orgname{National University of Defense Technology}, \orgaddress{\street{}, \city{ChangSha}, \postcode{410073}, \state{}, \country{China}}}


\abstract{Federated Learning (FL) is a distributed machine learning framework in communication network systems. However, the systems' Non-Independent and Identically Distributed (Non-IID) data negatively affect the convergence efficiency of the global model, since only a subset of these data samples are beneficial for model convergence. In pursuit of this subset, a reliable approach involves determining a measure of validity to rank the samples within the dataset. In this paper, We propose the BHerd strategy which selects a beneficial herd of local gradients to accelerate the convergence of the FL model. Specifically, we map the distribution of the local dataset to the local gradients and use the Herding strategy to obtain a permutation of the set of gradients, where the more advanced gradients in the permutation are closer to the average of the set of gradients. These top portion of the gradients will be selected and sent to the server for global aggregation. We conduct experiments on different datasets, models and scenarios by building a prototype system, and experimental results demonstrate that our BHerd strategy is effective in selecting beneficial local gradients to mitigate the effects brought by the Non-IID dataset, thus accelerating model convergence.}




\maketitle

\section{Introduction}

Federated learning is a paradigm of distributed machine learning that allows the decentralized clients to train a global model using their own private data, thus alleviating the problems of data silos and user privacy \cite{konevcny2016federated, liu2022distributed}. In FL, the clients' availability and decentralization bring great communication overhead to the network, and it can be reduced by increasing the local computing \cite{mcmahan2017communication}. The local computing of FL generally adopts the optimization method of Stochastic Gradient Descent (SGD) \cite{wang2021cooperative}, which requires the Independent and Identically Distributed (IID) data to ensure the unbiased estimation of the global gradient in the training process \cite{dvurechensky2018computational, lei2019stochastic, harvey2019tight, li2019hpdl, sungradient, sun2022decentralized}. Since the clients' local data are obtained from the local environment and usage habits, the generated datasets are usually Non-IID due to differences in size and distribution \cite{aledhari2020federated}. The Non-IID datasets and multiple local SGD iterations will bring a drift to the global model in each communication round, which significantly reduces the FL performance and stability in the training process, thus requiring more communication rounds to converge or even fail to converge \cite{sattler2019robust, luo2021no, zhang2021federated}. Therefore, it becomes a key challenge to accelerate FL convergence by reducing the negative impact of Non-IID data.

A common strategy is adapting the model to local tasks, which creates a better initial model by local fine-tuning \cite{fallah2020personalized,t2020personalized,finn2017model,smith2017federated, corinzia2019variational, chen2018federated,zhu2021data,lin2020ensemble,peng2019federated}. In addition to the optimization of FL strategies by adapting the model, the broader focus aims at improving data distribution from an optimization perspective of the data, thus effectively improving the generalization performance of the global model \cite{zhu2021federated,yoon2020fedmix,zhang2018mixup,zhou2022domain, back2022mitigating,wang2020optimizing,yang2021federated,zhao2018federated,tuor2021overcoming, yoshida2019hybrid}. Further, by mapping local data to local gradients via model parameters, what can be found is that the local gradients influence the bias of the global aggregation. So some studies have optimized the convergence process of FL by adjusting these local gradients \cite{cheng2022aafl,wang2020tackling,karimireddy2020scaffold}. Through this research, we identified a pathway to accelerate FL convergence: mitigating the influence of Non-IID data through modifications to the local gradient.

Let us broaden our perspective beyond FL. Here is an illuminating statement: only a significant portion of the training samples play an important role in the generalization performance of the model, and thus dataset pruning can be used to construct the smallest subset from the entire training data as a proxy training set without significantly sacrificing the model's performance \cite{paul2021deep,yang2022dataset}. Translating this concept to FL yields an innovative insight: within a Non-IID dataset, only a subset of samples contribute significantly to the convergence of the global model. By isolating and utilizing the gradients derived from the training on these samples—termed "beneficial gradients"—for global updates, one can enhance the model's generalization capabilities and expedite its convergence. Hence, the primary focus of this paper is to examine the methodology for selecting these "beneficial gradients" and to validate the efficacy of this approach.

In this paper, we first propose a new federated learning optimization strategy called BHerd. BHerd orders local gradients using the herding strategy in GraB \cite{lu2022grab}, and then performs the cropping which means selecting only the top portion (regarded as "beneficial gradients") in the permutation of the gradients. Subsequently, we will delineate the architecture and operational dynamics of our prototype system framework and the BHerd strategy in detail.

\begin{figure}[!t]
	\centering
	\includegraphics[width=1.0\textwidth]{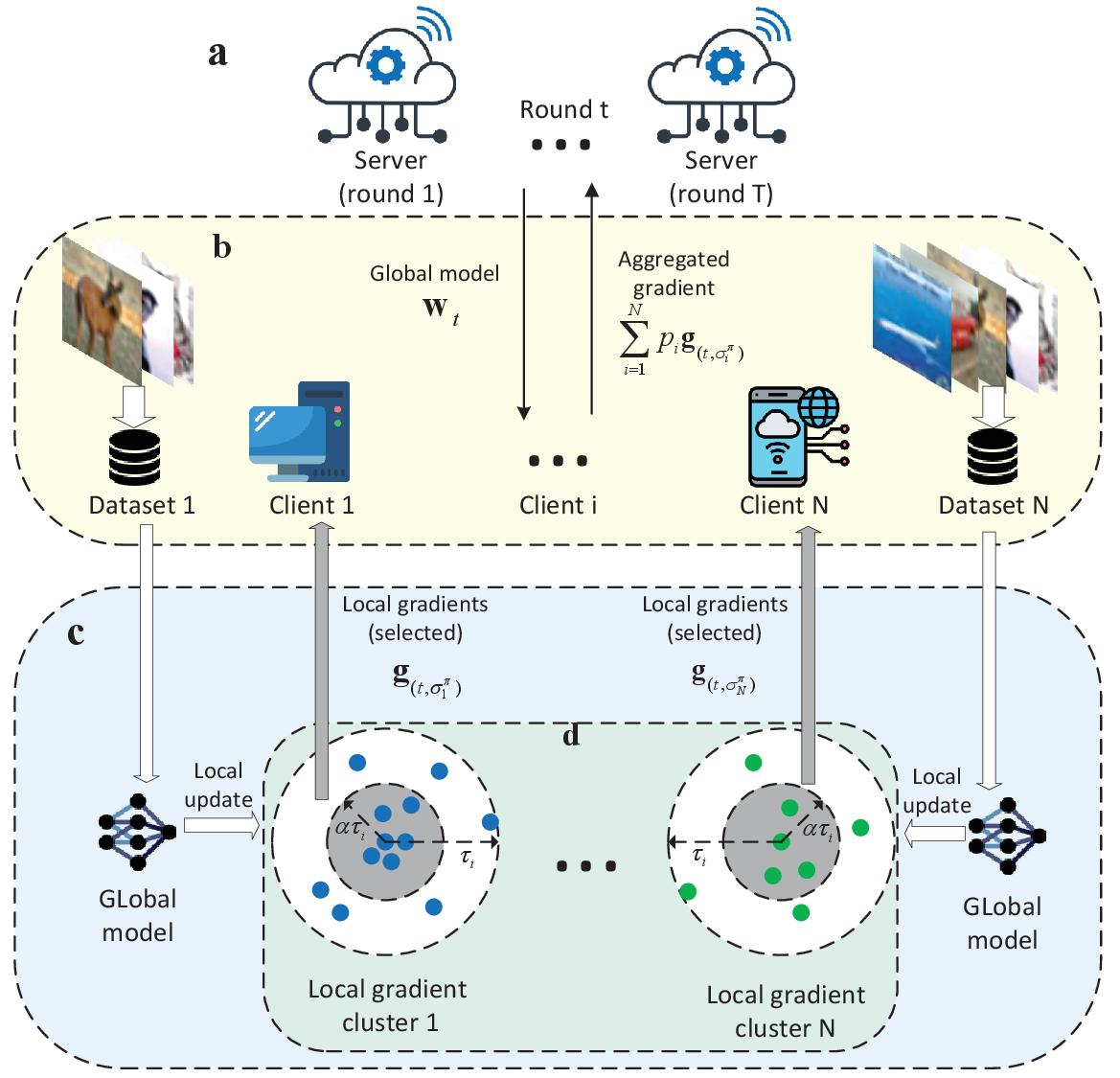}
	\caption{\footnotesize{
\textbf{Systems for herding local gradients in FL.}
(a) Server-side global updates. During the $t$-th ($t\in[1, T]$) round of global training, the server disseminates the global model $ \boldsymbol{w}_{t}$ to all participating clients, followed by executing a global update upon receipt of the gradients learned by the clients.
(b) Client Preparation Phase. Upon receiving the global model $ \boldsymbol{w}_{t}$, the client $i$ ($i\in[1, N])$ invokes the local dataset and initiates training preparation in accordance with the predefined parameters ($E=1, B=100, N=5, T=500, \alpha=0.5, \eta=10^{-4}$).
(c) Local update to obtain local gradient clusters. During this phase, the client $i$ undertakes $\tau_i$ iterative steps of local updates applied to the global gradient, symbolized as $ \boldsymbol{w}_{t}$, for the purpose of generating $\tau_i$ distinct local gradients.
(d) Selection of beneficial partial local gradients. The procured local gradients undergo a selection process as delineated by Equation \eqref{local-obj}. For enhanced visualization and comprehension, these gradients are represented on a two-dimensional plane within the figure. Points positioned nearer to the circle's center signify a reduced deviation from the mean gradient, thus receiving a higher ranking. Subsequently, the quantity of local gradients is refined to the foremost $\alpha\tau_i$, constituting the outcome of the final selection.
}}
\label{system}
\end{figure}

\subsection{BHerd's Objective}
\label{control}
For convenience, we assume that all vectors are column vectors in this paper. Assume we have $ N $ clients with local train datasets $ \mathcal{D}_{1}, \mathcal{D}_{2}, ..., \mathcal{D}_{i}, ..., \mathcal{D}_{N} $, where $ i $ denotes the client index and the training sample $\boldsymbol{x}_{i} \in \mathcal{D}_{i}$. We use $ \boldsymbol{w}$ to denote the model parameters and use $ \left\|.\right\| $ to denote the $ L_2 $ norm. It should be mentioned that we use mini-batch gradient descent with batch size $B\geq1$ and local epoch $E\geq0$, so the differentiable (loss) function $F:\mathbb{R}^{d}\to\mathbb{R}$ is minimized on a set of data batch $\boldsymbol{x}$. 

We set the size of the total training dataset, batch size and training epoch as constant with the value of $ |\mathcal{D}| = 6 \times 10^{4} $, $B=100$ and $E=1$, respectively. It is important to mention that we assume in this paper that all clients participate in training, and this assumption can easily be generalized to pick a different fraction of clients to participate in each round of training.

When fastening the number of clients $N$, the local data size $|\mathcal{D}_i|$ will have different fixed values in the different client $i$, and then we can get a fixed value for $ \tau_{(t, i)} = \lfloor E\frac{|\mathcal{D}_i|}{B} \rfloor $ to the clients for each round (except for the first baseline with $\tau= \lfloor E\frac{|\mathcal{D}|}{B} \rfloor$ in one client). Unless otherwise specified, we set the control parameter $ \alpha $ to a fixed value $ \alpha = 0.5 $ for all rounds. We manually select the total round $ T =500 $ and the learning rate fixed at $ \eta = 1 \times 10^{-4} $, which is acceptable for our learning process. Except for the instantaneous results, the others are the average results of 10 independent runs.

We sort the gradients collected to be close to the average gradient $\boldsymbol{\mu}_{(t,\sigma_i)}$ ($\sigma_i=\{1,2,...,\tau_i\}$ denotes a permutation (ordering) of data batch $\boldsymbol{x}$) and pick only the top $\alpha\tau_{i}, \alpha\in\left( 0,1\right] $ (rounding when not an integer) gradients to form a new permutation. Based on these, the objective of BHerd is formalized in mathematical notation as follows:
\begin{equation}
	\label{local-obj}
	\arg\min_{\sigma_i^\pi}\left\|\sum_{\lambda=1}^{\alpha\tau_i}\left(\nabla F_i\left(\boldsymbol{w}_{(t,i)}^\lambda;\boldsymbol{x}_{\sigma_i}\right)-\boldsymbol{\mu}_{(t,\sigma_i)}\right)\right\|,
\end{equation}
where $\sigma_i^{\pi}=\{\pi_1,\pi_2,...,\pi_{\alpha\tau_i}\}$ denotes a permutation (ordering) of local gradients. The purpose of Equation \eqref{local-obj} is to characterize the entire set of gradients in terms of partial of them, thus mitigating the effects of long-tailed gradients and improving the generalization ability of the model (Fig. \ref{system}).

\subsection{Model and Datasets}
\label{data-case1}
We evaluate the training of two different models on two different datasets, which represent both small and large models and datasets. The models include squared-SVM (we refer to as SVM in short in the following) and deep convolutional neural networks (CNN). The SVM has a fully connected neural network, and outputs a binary label that corresponds to whether the digit is even or odd. The CNN has two $ 5 \times 5 \times 32 $ convolution layers, two $ 2 \times 2 $ MaxPoll layers, a $ 1568 \times 256 $ fully connected layer, a $ 256 \times 10 $ fully connected layer, and a softmax output layer with 10 units. The SVM model uses Squared Hinge Loss and the CNN model uses Cross-Entropy Loss, as detailed in \cite{wang2019adaptive}.

SVM is trained on the original MNIST (which we refer to as MNIST in short in the following) dataset \cite{lecun1998gradient}, which contains the gray-scale images of $ 7 \times 10^{4} $ handwritten digits ($ 6 \times 10^{4} $ for training and $ 10^{4} $ for testing).

CNN is trained using SGD on two different datasets: the MNIST dataset and the CIFAR-10 dataset \cite{krizhevsky2009learning}, and the CIFAR-10 dataset includes $ 6 \times 10^{4} $ color images ($ 5 \times 10^{4} $ for training and $ 10^{4} $ for testing) associated with a label from 10 classes. A separate CNN model is trained on each dataset, to perform multi-class classification among the 10 different labels in the dataset.

\subsection{Baselines}
\label{baseline}
We compare with the following baseline approaches: 

\textbf{Centralized SGD.} The entire training dataset is stored on a single client and the model is trained directly on that client using a standard SGD and random reshuffling protocol \cite{mishchenko2020random} with $E=500$.

\textbf{FedAvg.} The standard FL approach \cite{mcmahan2017communication} that all clients use SGD and random reshuffling protocol with local epoch $E=1$ and global round $T=500$.

\textbf{GraB-FedAvg.} The online Gradient Balancing (GraB) \cite{lu2022grab} operation replaces the herding operation in the BHerd strategy with local epoch $E=1$ and global round $T=500$. In the GraB operation, when the current local gradient is received, it is roughly ordered based on the average of all the previous gradients to reduce the storage overhead and computational complexity.

\textbf{FedNova.} A novel FL approach \cite{wang2020tackling} that all clients use SGD and random reshuffling protocol with local epoch $E=1$ and global round $T=500$. Instead of ordering and selecting the local gradients, the FedNova algorithm simply performs an averaging operation and then scales the global gradients.

\textbf{SCAFFOLD.} A novel FL approach \cite{karimireddy2020scaffold} that all clients use SGD and random reshuffling protocol with local epoch $E=1$ and global round $T=500$. SCAFFOLD uses control variates (variance reduction) to correct for the ‘client-drift’ in its local updates.

A detailed implementation of the Herd-FedAvg and Grab-FedAvg algorithms can be found in Appendix \ref{alg-result}. For a fair comparison, we set the result of the first baseline as the standard model under its epoch budget, and use this model for evaluating the convergence of other trained models.

\subsection{Simulation of Dataset Distribution}
\label{sec-distribute}
To simulate the dataset distribution, we set up two different Non-IID cases and a standard IID case.

\textbf{Case 1 (IID).} Each data sample is randomly assigned to a client, thus each client has uniform (but not full) information.

\textbf{Case 2 (Non-IID).} All the data samples in each client have the same label, which means that the dataset on each client has its unique features.

\textbf{Case 3 (Non-IID).} Data samples with the first half of the labels are distributed to the first half of the clients as in Case 1, the other samples are distributed to the second half of the clients as in Case 2.

\section{Results}
\label{result}
We first evaluate the general performance of the BHerd strategy, thus we conducted experiments on a networked prototype system with five clients. The prototype system consists of five parallel distributed processes (clients) and one process for parameter initialization, distribution, and handling (server), which are implemented on an instance configured with a 24-core Intel(R) i9-13900K 5.8GHz CPU, 32GB memory and an NVIDIA GeForce RTX 4070 Ti GPU. The server has an aggregator and implements global update and offline herding steps of FL, and the five clients have model training with local datasets and different simulated dataset distributions of IID and Non-IID. We set up that in BHerd strategy, each client trains the local dataset for 1 epoch ($E=1$) in a global update round, with a total of 100 global rounds ($T=500$).

\subsection{BHerd's combination on the classic FedAvg structure}
\label{loss and acc}

\begin{figure}[!t]
\centering

\subfloat[]{\includegraphics[width=1.0\textwidth]{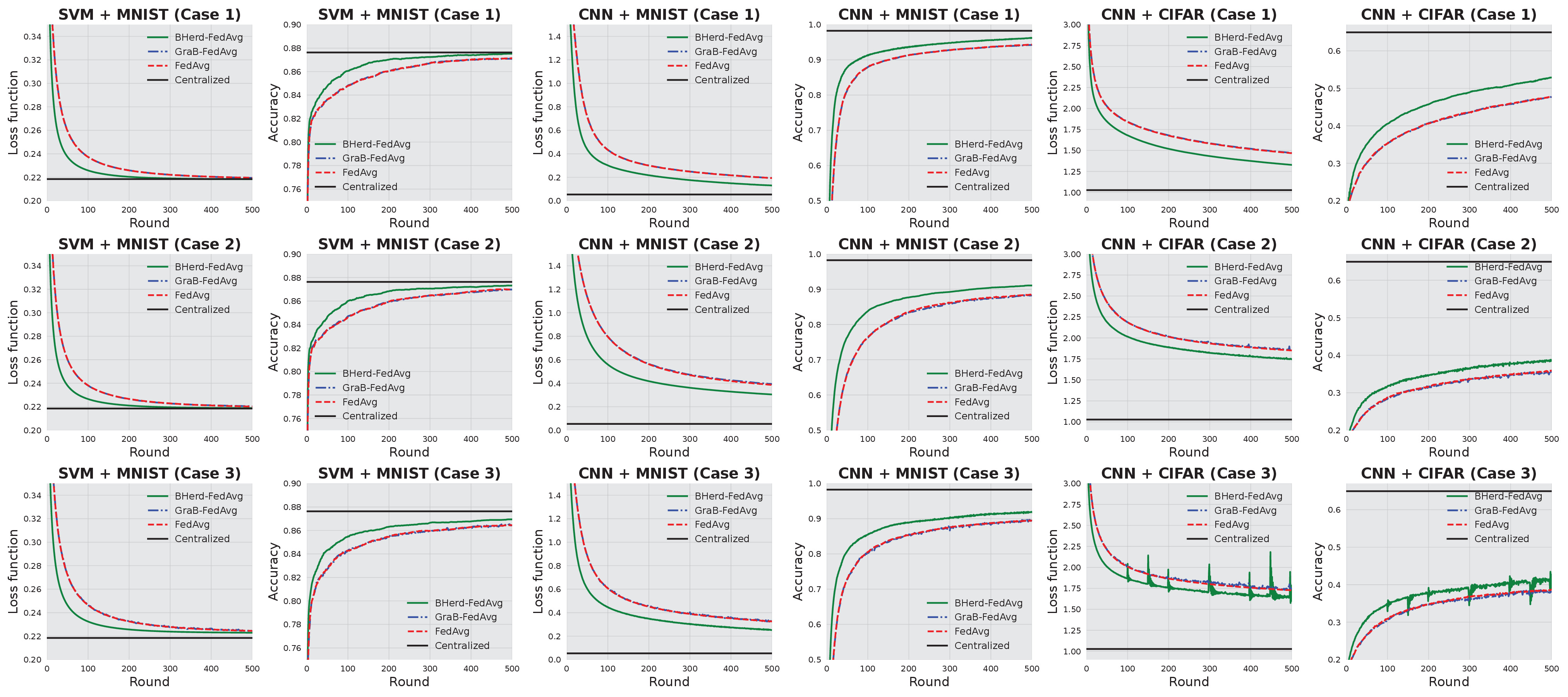}\label{lossacc:bherd-grab}}

\subfloat[]{\includegraphics[width=1.0\textwidth]{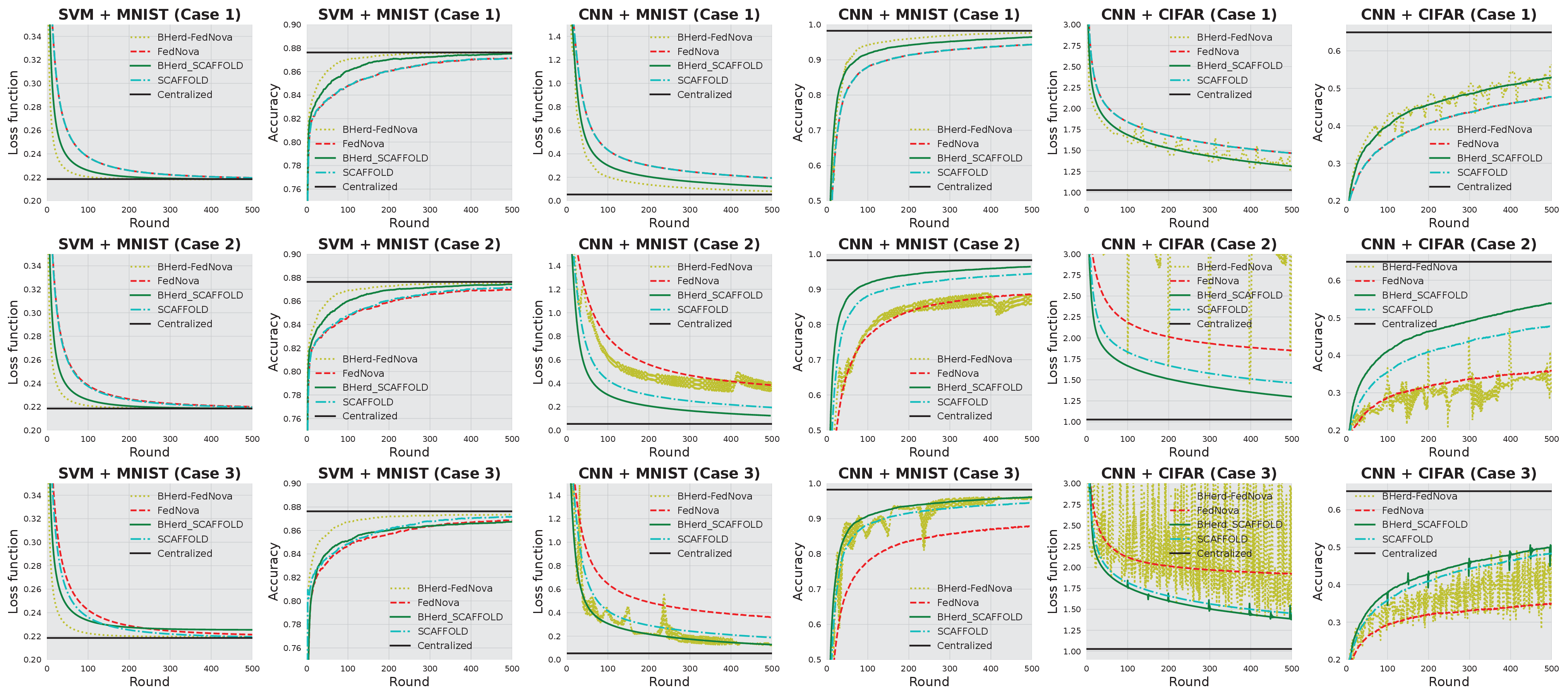}\label{lossacc:nova-scafflod}}

\caption{\footnotesize{
\textbf{Introducing the BHerd strategy to enhance each baseline in \textbf{Case 1-3}.}
In all figures requiring a quantitative assessment of the model, metrics are represented in terms of loss function values and prediction accuracy values, discernible from the vertical axis. 
(a) Comparison of BHerd and GraB method in the prototype system. The black line delineates the outcome of the centralized gradient descent previously discussed, signifying the model's convergence optimum. The red curve signifies the conventional FedAvg algorithm, serving as the optimization benchmark. The green and blue curves depict the integration of the BHerd and GraB methods into the FedAvg algorithm, respectively.
(b) Extension of the BHerd strategy to popular algorithms. The red and cyan curves signify the performance of two recently prominent algorithms, FedNova and SCAFFLOD, respectively. The yellow and green curves illustrate the outcomes of integrating the BHerd strategy with the aforementioned algorithms.
}}
\label{fig-loss-acc}
\end{figure}

In our first set of experiments, the SVM and CNN models were trained on the prototype system with \textbf{Case 1-3} and five clients ($N=5$). The initial objective of this study is to ascertain the differential efficacy in the convergence rates of federated models facilitated by the GraB sorting algorithm delineated in \cite{lu2022grab} as compared to the BHerd strategy. As shown in Figure \ref{lossacc:bherd-grab}, incorporating the BHerd strategy into the FL framework markedly enhances the speed at which convergence is attained, as well as the ultimate predictive precision of the model, surpassing the performance of the traditional FedAvg algorithm. Conversely, despite the GraB algorithm's superior performance in computational complexity and storage requirements, it falls short in enhancing the convergence rate of FedAvg. This discrepancy underscores the stochastic nature of the GraB algorithm within our devised strategy for the selection of local gradients, thereby affirming the premise that while GraB may excel in certain operational efficiencies, it does not necessarily contribute to the optimization of convergence within the Federated Learning paradigm. This comparative analysis not only highlights the distinct advantages of implementing BHerd in federated settings but also delineates the critical factors influencing the effectiveness of algorithms in enhancing the convergence and accuracy of federated models.

It is imperative to acknowledge that in the context of `CNN+CIFAR' evaluations within Case3, the performance metrics associated with BHerd exhibit fluctuations across successive rounds. This variability stems from the algorithm's engagement with datasets characterized by a significant degree of Non-IID, wherein BHerd selects local gradients that, during certain rounds, diverge markedly from the overarching global target. This tendency is particularly pronounced in the context of CNN models, which exhibit heightened sensitivity to the CIFAR dataset's samples. While it is technically feasible to mitigate these oscillations by excluding training outcomes that deviate from expected patterns at each round, our commitment to presenting an unvarnished depiction of the algorithm's performance has led us to faithfully document these results in Figure \ref{fig-loss-acc}. This approach not only ensures transparency but also facilitates a comprehensive understanding of the dynamic interplay between algorithmic selection mechanisms and dataset characteristics, thereby enriching the discourse on the optimization of federated learning strategies in complex data environments.

\subsection{BHerd's extension to the popular algorithms}
Despite the fact that our theoretical validation of the BHerd strategy was confined within the bounds of the FedAvg framework, as corroborated by the empirical evidence presented in preceding experiments, the potential for generalization of the BHerd approach extends to encompass other leading Federated Learning (FL) algorithms, notably FedNova \cite{wang2020tackling} and SCAFFOLD \cite{karimireddy2020scaffold}, which have gained prominence in recent years. 

As depicted in Figure \ref{lossacc:nova-scafflod}, the integration of our BHerd strategy into the frameworks of FedNova and SCAFFOLD results in varying degrees of improvement in convergence speeds for both algorithms. Notably, within the context of cases1-3, FedNova exhibits a marked enhancement when applied to the SVM model, whereas the improvements with SCAFFOLD are more pronounced in its application to the CNN model. Particularly noteworthy is the observation that the performance of FedNova, when training the CNN model, exhibits significant fluctuations across rounds. A feasible explanation for this phenomenon lies in the inherent sensitivity of CNN models to variations in the global learning rate. Our BHerd strategy, by eliminating certain local gradient information, inadvertently contributes to these fluctuations. Specifically, FedNova's approach to adjusting the step size—by reducing the local gradient step size while increasing the global step size—tends to amplify the observed oscillations highlighted in prior experiments. In contrast, SCAFFOLD, which focuses on correcting both the local and global gradients without significantly altering their respective magnitudes, introduces additional gradient information that acts to dampen the oscillation effect. 

Collectively, the BHerd strategy substantiates its efficacy in relation to conventional methodologies, evidenced by its convergence towards minimal loss metrics and enhanced accuracy across standard models, datasets, and data distribution frameworks. Owing to the intricate challenges associated with evaluating CNN models, this study henceforth narrows its focus to the examination of the SVM model utilizing the MNIST dataset and BHerd-FedAvg framework, aiming to elucidate further on the prototype system's performance. Subsequently, an analysis delineating the correlation between the algorithm's efficiency and the configuration of pertinent hyperparameters will be conducted.

\subsection{Sensitivity to the beneficial local gradients}
\label{alpha}

\begin{figure}[!t]
\centering

\subfloat[]{\includegraphics[width=1.0\textwidth]{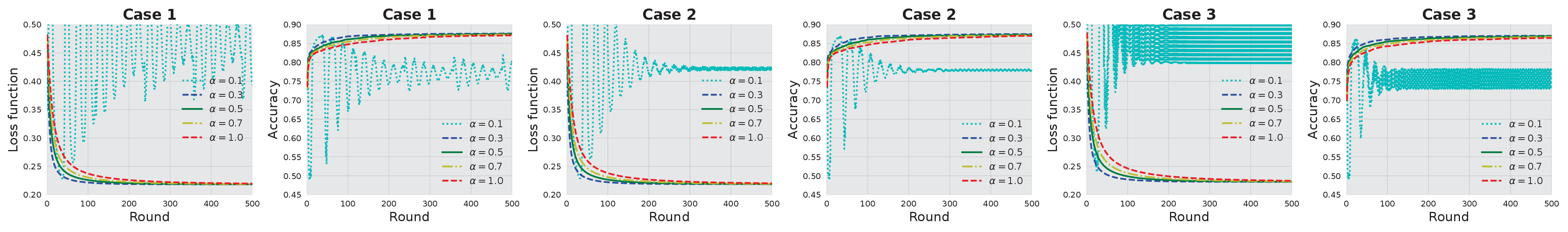}\label{fig-alpha}}

\subfloat[]{\includegraphics[width=1.0\textwidth]{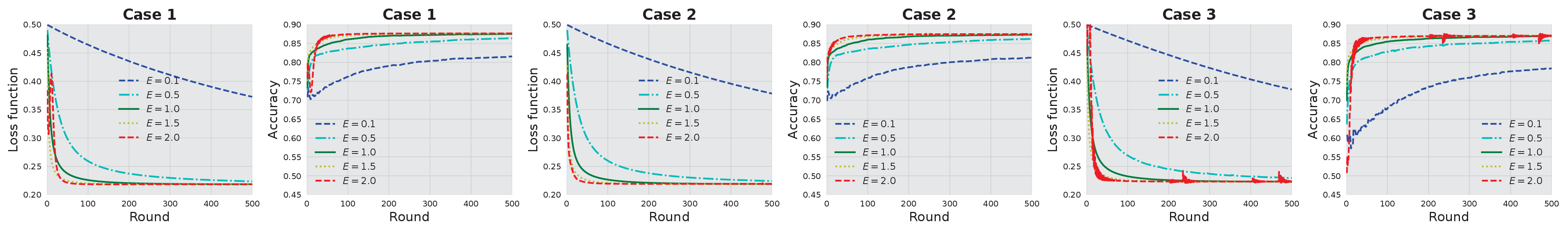}\label{fig-epoch}}

\subfloat[]{\includegraphics[width=1.0\textwidth]{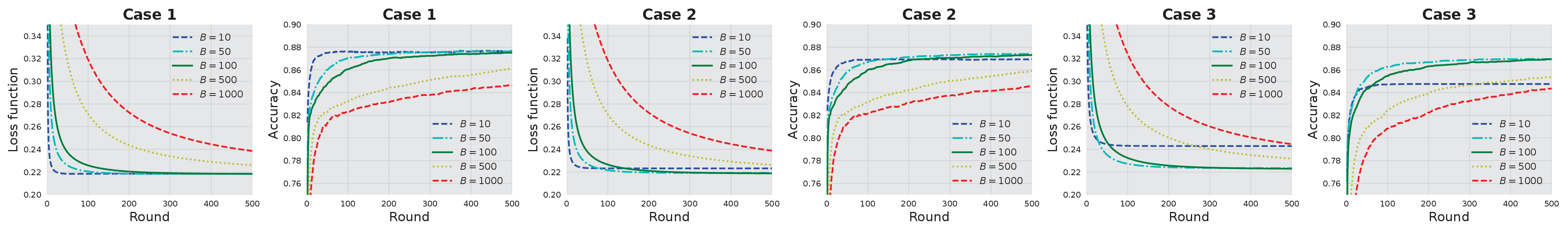}\label{fig-batch}}

\subfloat[]{\includegraphics[width=1.0\textwidth]{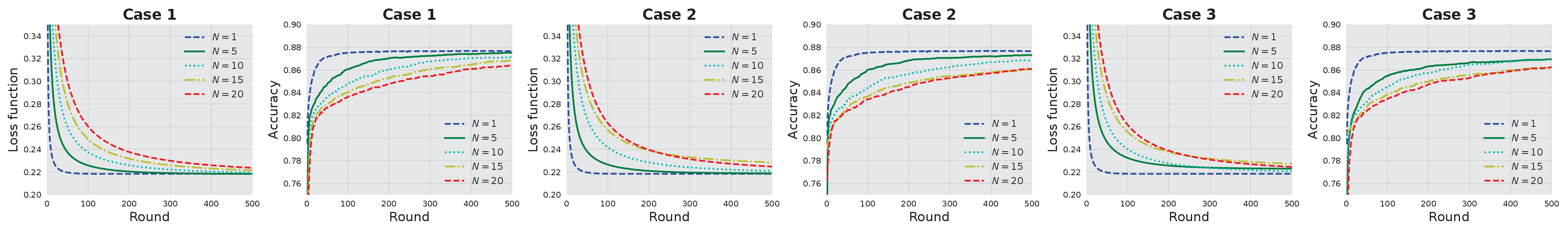}\label{fig-node}}

\caption{\footnotesize{
\textbf{Influence of hyperparameters on global model convergence.}
Consistent with the prototype system delineated in Section \ref{loss and acc}, we present the outcomes of deploying the `SVM + MNIST' configuration across five clients under \textbf{Case 1-3}; Each subplot is designed to illustrate the impact of a single hyperparameter(all inputs are constrained within their respective value ranges), with all other parameters held constant, thereby facilitating a clear analysis of its individual effect.
(a) Acceleration of model convergence is facilitated by the scaling of beneficial gradients. The legends for the individual curves are arranged from top to bottom as follows: $\alpha = 0.1, 0.3, 0.5, 0.7, 1.0$.
(b) The number of epochs executed per round influences model convergence. The legends for the individual curves are arranged from top to bottom as follows: $E = 0.1, 0.5, 1.0, 1.5, 2.0$.
(c) Batch size of the training data impacts model convergence. The legends for the individual curves are arranged from top to bottom as follows: $E = 10, 50, 100, 500, 1000$.
(d) The specificity of the prototype system is influenced by the number of clients.
The legends for the individual curves are arranged from top to bottom as follows: $N = 1, 5, 10, 15, 20$.
}}
\label{fig-parameter}
\end{figure}

The quantitative outcomes pertaining to the loss function and classification accuracy are depicted in Figure \ref{fig-alpha}. Upon comparative analysis with the `SVM + MNIST' subplots illustrated in Figure \ref{lossacc:bherd-grab}, it is discernible that the convergence trajectory of the algorithm under scrutiny bears resemblance to that of the FedAvg algorithm at $\alpha = 1.0$. This observation suggests that the FedAvg algorithm can be regarded as a particular instantiation of the algorithm proposed herein. Notably, a reduction in the value of $\alpha$ from 1.0 to 0.3 engenders an acceleration in the convergence rate of the BHerd strategy, attributable to the mitigation of gradient deviations on each client, as delineated by Equation \eqref{local-obj}. Contrariwise, diminishing $\alpha$ further to 0.1 results in the algorithm's failure to converge, indicative of the local gradient's incapacity to contribute meaningful parameters for global aggregation. Consequently, to circumvent the issues of non-convergence and oscillations observed respectively in Figures \ref{fig-alpha} and \ref{lossacc:bherd-grab}, an $\alpha$ value of 0.5 is posited as the optimal threshold. This threshold is instrumental in characterizing the overall local gradient by incorporating a selectively proportioned segment, thereby ensuring the algorithm's effective performance.

\subsection{Variation in Local Epoch per Round}
\label{epoch}

On a global level, the BHerd strategy does a scaling transformation: $\tau_i \to \alpha\tau_i $. And according to the computation on $ \tau_{(t, i)} = \lfloor E\frac{|\mathcal{D}_i|}{B} \rfloor $ in Section \ref{control}, the parameters that have an effect on $\tau_i$ are $E$, $|\mathcal{D}_i|$ (which is directly determined by $N$) and $B$. Therefore, we do ablation experiments on the three parameters $E$, $B$ and $N$, respectively.

The derived results concerning the loss function and classification accuracy are illustrated in Figure \ref{fig-epoch}. It is observed that as $E$ ascends from 0.1 to 2, there is an acceleration in the convergence of the depicted curves. This acceleration can be attributed to the increase in advantageous gradients at each client, thereby facilitating a more rapid convergence of the model. Nevertheless, at $E = 2$, the model's convergence trajectory exhibits minor oscillations across rounds. This phenomenon can be attributed to the accumulation of local gradients, divergent from the global objective, resultant from extensive local training sessions. Such divergence amplifies the likelihood of our BHerd strategy selecting these aberrant gradients in certain rounds, consequently impeding. Furthermore, it is observed that the computational complexity of BHerd escalates exponentially with an increase in $E$, thereby necessitating the selection of $E = 1$ as a pragmatic value to balance efficiency and performance.

\subsection{Variation in Local Batch Size}

The impact of these variations on the loss function and classification accuracy is documented in Figure \ref{fig-batch}. While $E$ modulates the quantity of selected gradients by adjusting the aggregate gradient count per round, $B$, in contrast, influences the total count of local gradients, denoted as $\frac{|\mathcal{D}_i|}{B}$, given fixed $E$ and $N$ (i.e., $|\mathcal{D}_i|$). Owing to the unique operational dynamics of the BHerd strategy, the selection of $B$ critically alters the distribution of local training datasets and, consequently, the number of local training gradients. This adjustment necessitates an optimal $B$ value, which varies distinctly across different \textbf{Case} scenarios.

As delineated in Figure \ref{fig-batch}, the identified optimal $B$ values for \textbf{Case 1-3} are 10, 10, and 50, respectively. It is observed that above these optimal values, an increment in batch size inversely affects the model's convergence rate and predictive accuracy. A conceivable rationale for this phenomenon posits that lower $B$ values enhance the likelihood of selecting more beneficial training features within local datasets through the herding operation, facilitating convergence to the local optimum on individual clients. Consequently, in scenarios where the training dataset is IID—thereby equating local optimization with global optimization—a reduced $B$ value can expedite convergence towards the global optimum, maintaining a minimal distance value as discussed in Section \ref{distance}. Conversely, in the presence of non-IID training data, it becomes imperative to adjust the $B$ value upwards to a level that sustains a low distance value, thereby augmenting the dataset's generalization capability and averting convergence to suboptimal local minima.

\subsection{Variation in the Number of Clients}
The results are shown in Figure \ref{fig-node}. When $N = 1$, the BHerd strategy is equivalent to centralized gradient descent, and thus its curves have the fastest convergence rate in \textbf{Case 1-3}, and next we analyze the impact of $N$ from 5 to 20.

Unlike $E$, which affects the number of beneficial gradients selected, and $B$, which affects the selection rate of beneficial gradients, $N$ affects the distribution of beneficial gradients by varying the size of $|\mathcal{D}_i|$. When $N$ is small and the data is more centralized, it is easier for the FL model to quickly converge to the local optimum for each client. This also explains that in \textbf{Case 1-3}, the curves with the fastest convergence and the best convergence results are $N = 5$.

\subsection{Distribution of selected local gradients and distance from the mean gradient}
\label{distance}

\begin{figure}[!ht]
\centering

\subfloat[]{\includegraphics[width=0.19\textwidth]{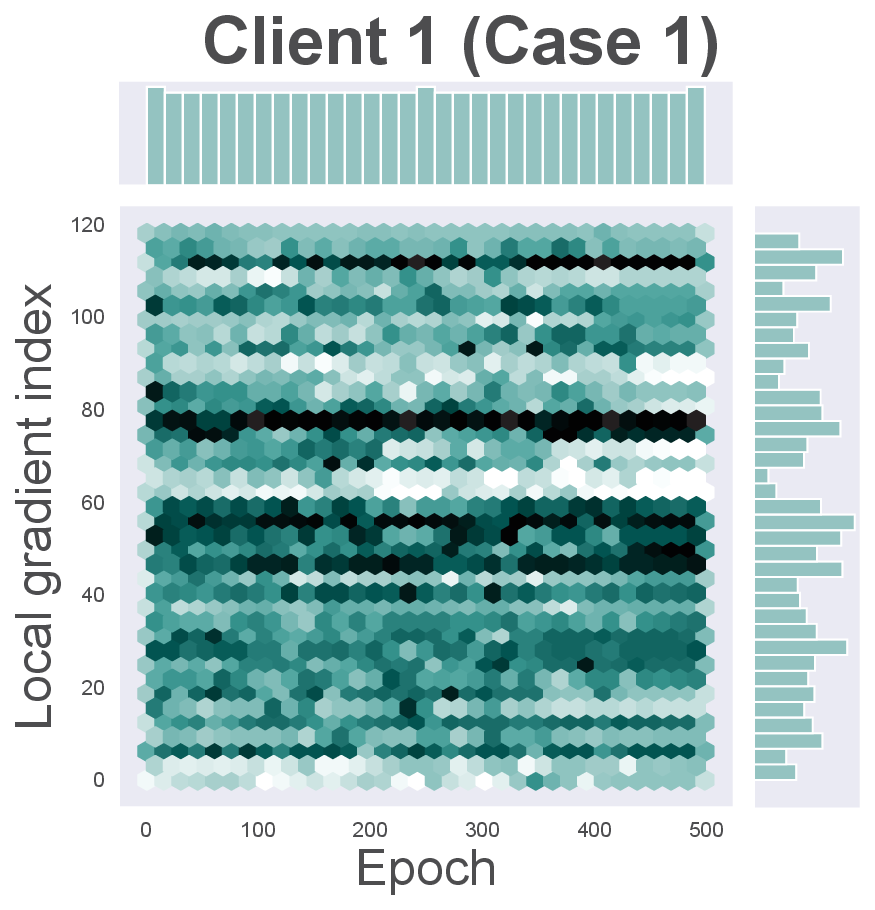}
\includegraphics[width=0.19\textwidth]{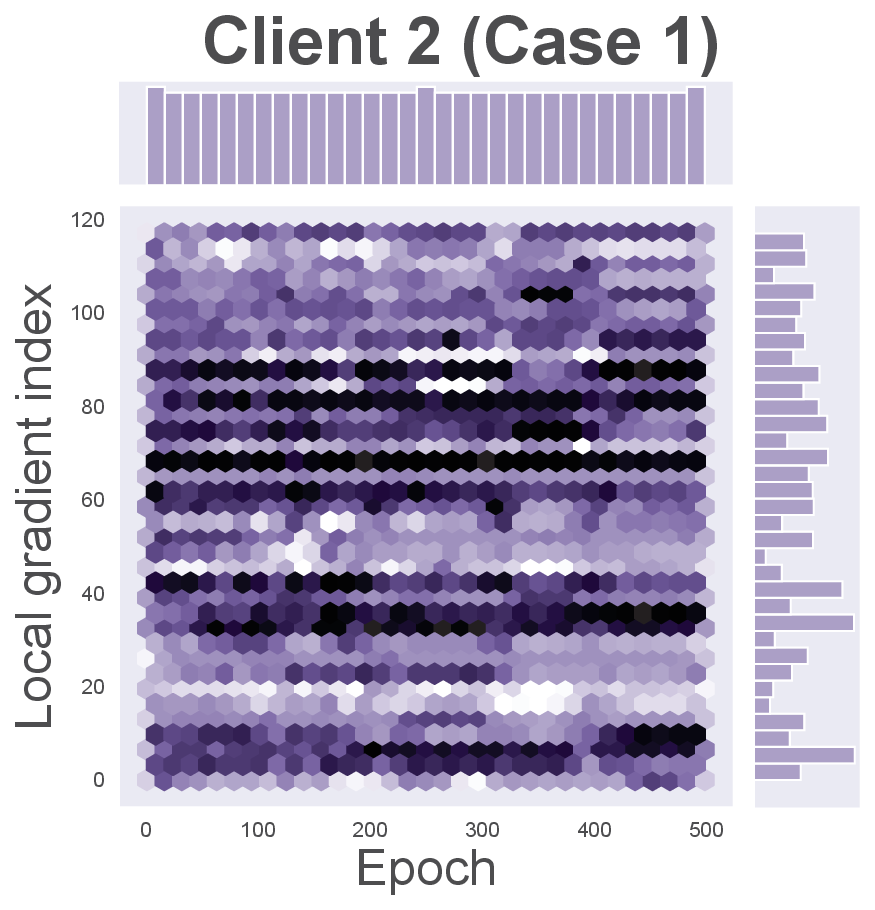}
\includegraphics[width=0.19\textwidth]{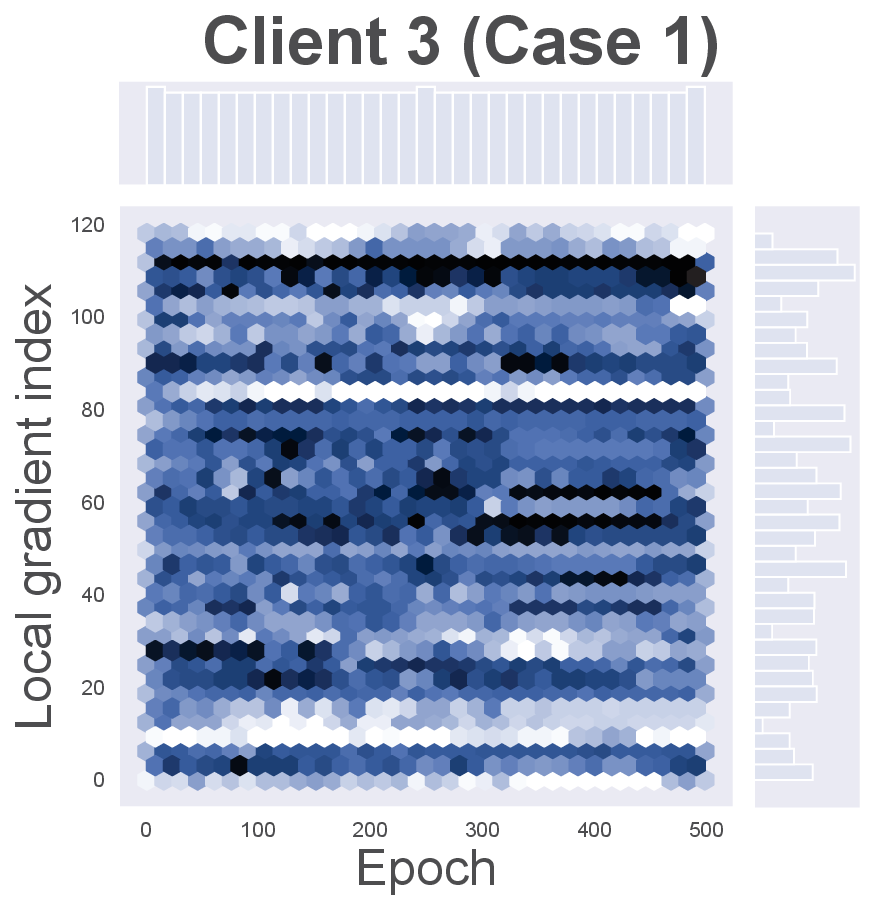}
\includegraphics[width=0.19\textwidth]{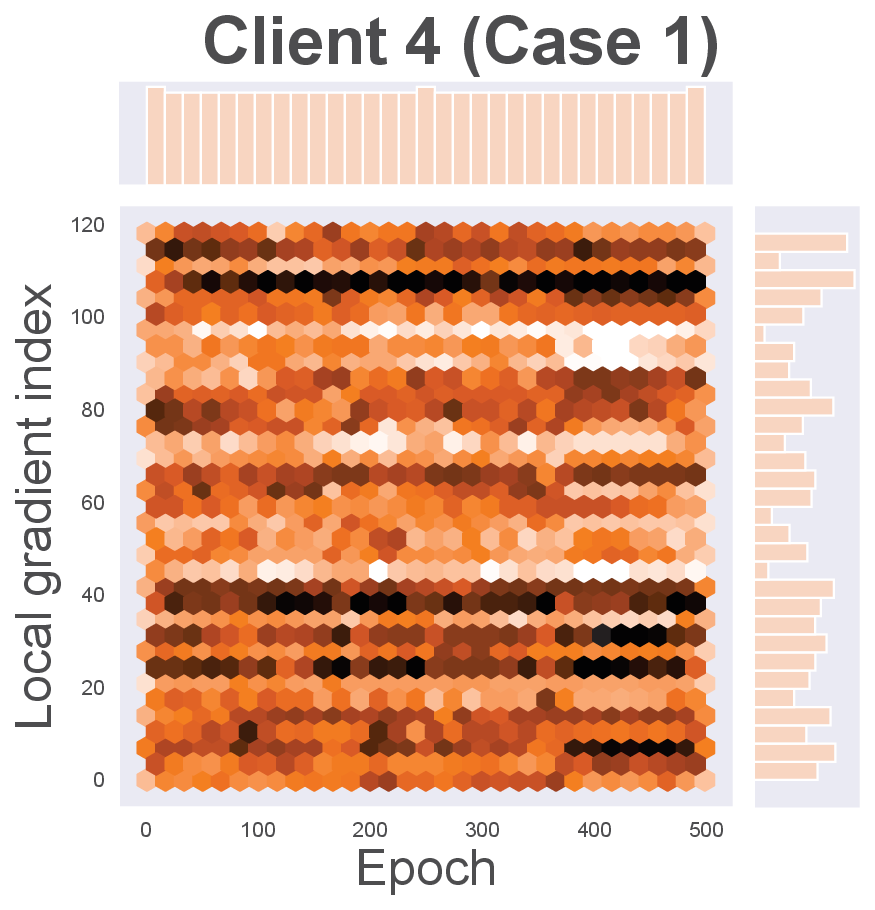}
\includegraphics[width=0.19\textwidth]{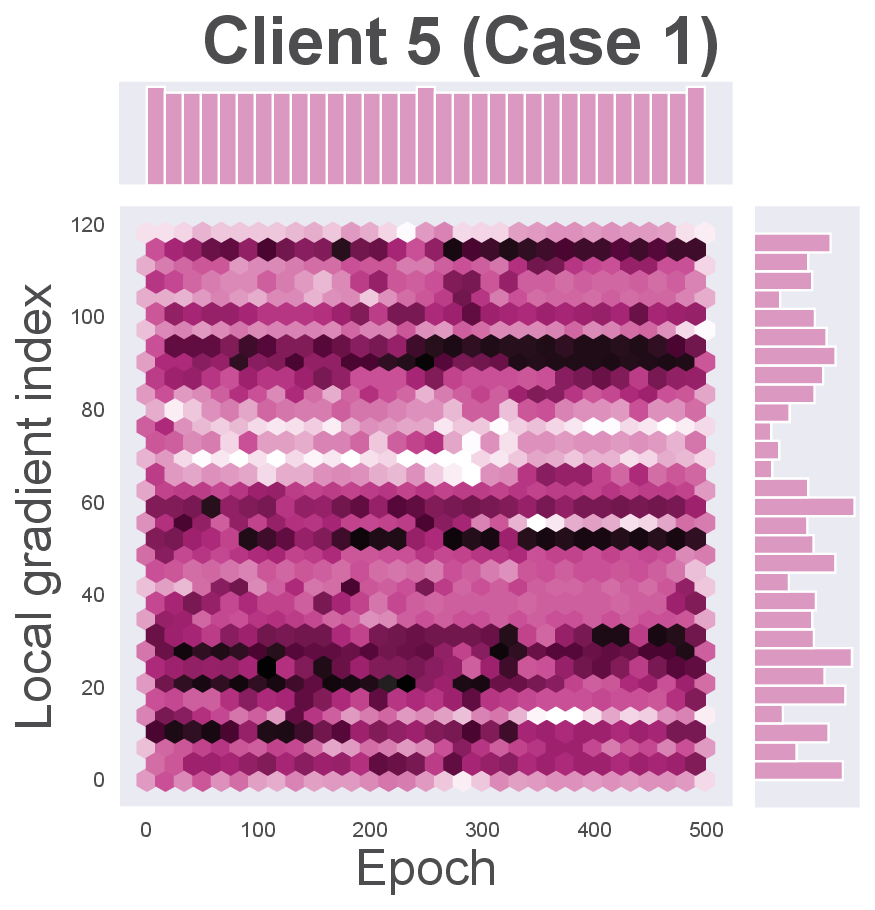}
\label{distribute_c1}}

\subfloat[]{\includegraphics[width=0.19\textwidth]{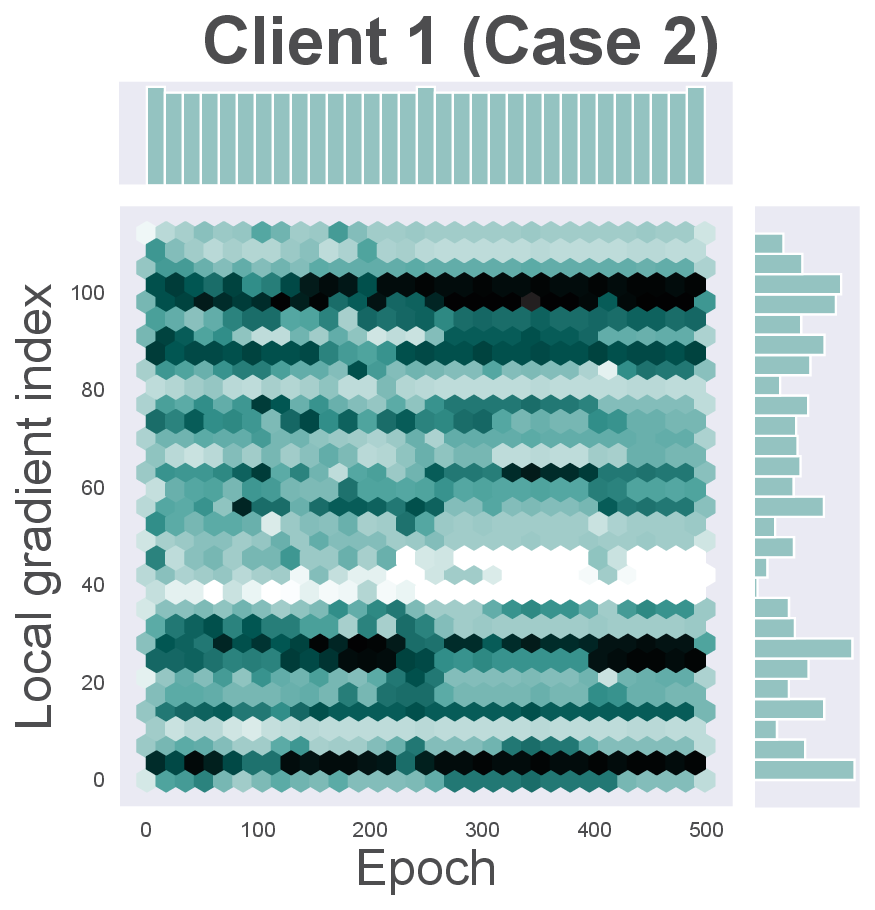}
\includegraphics[width=0.19\textwidth]{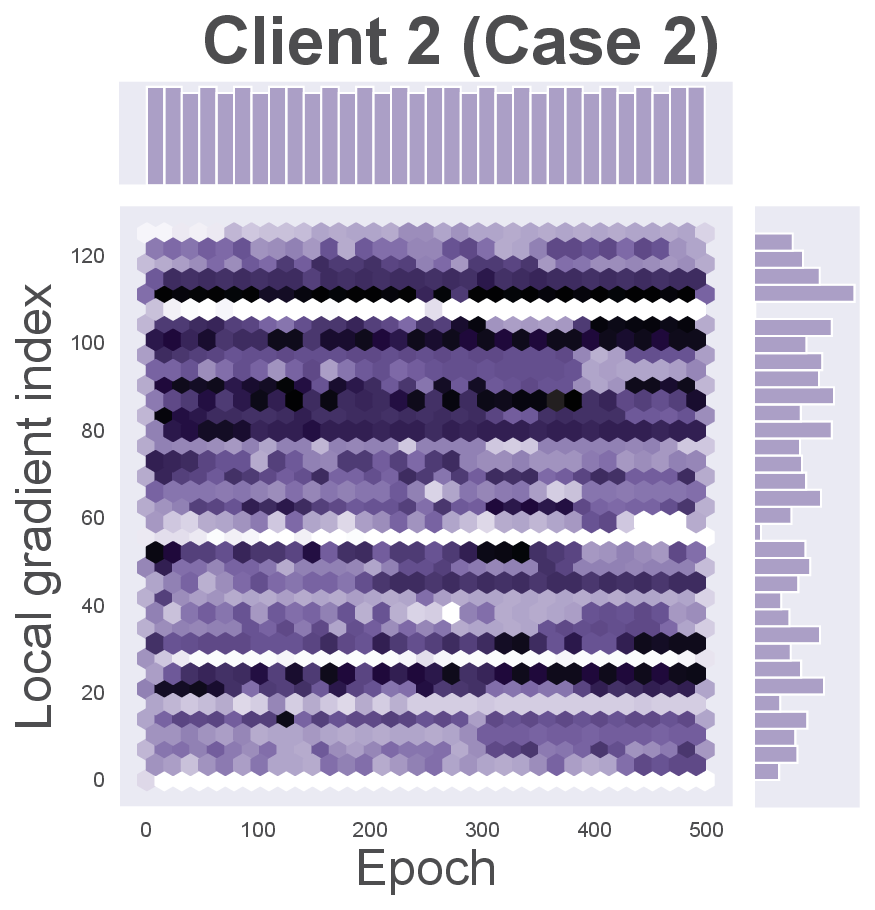}
\includegraphics[width=0.19\textwidth]{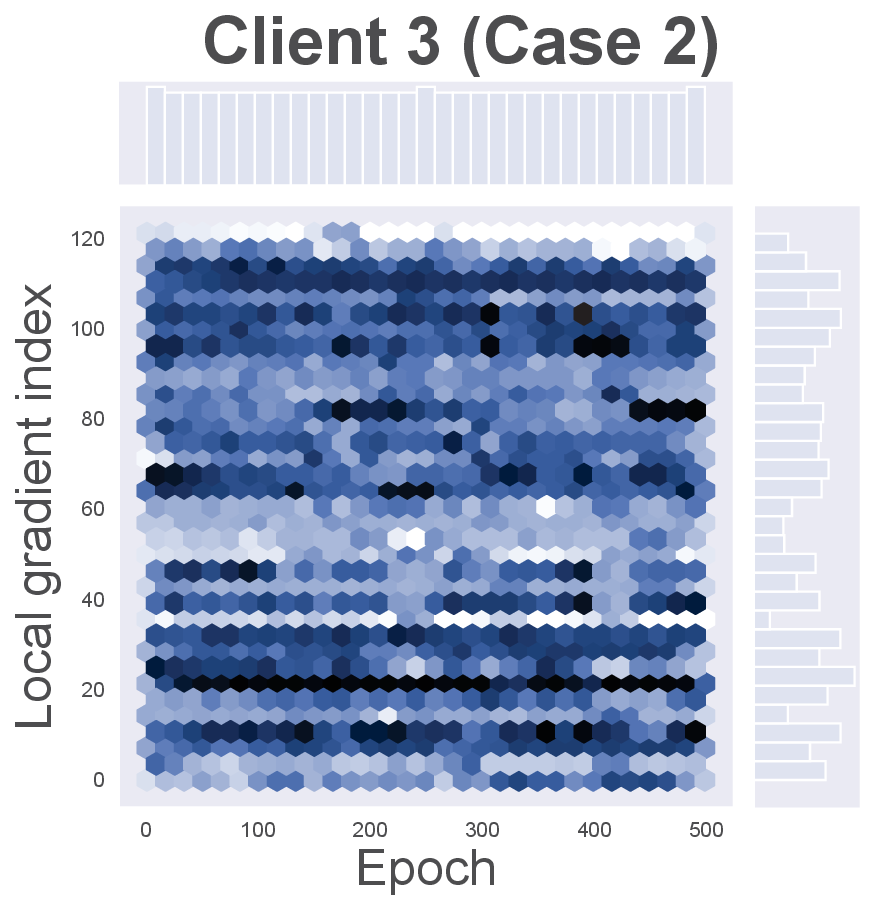}
\includegraphics[width=0.19\textwidth]{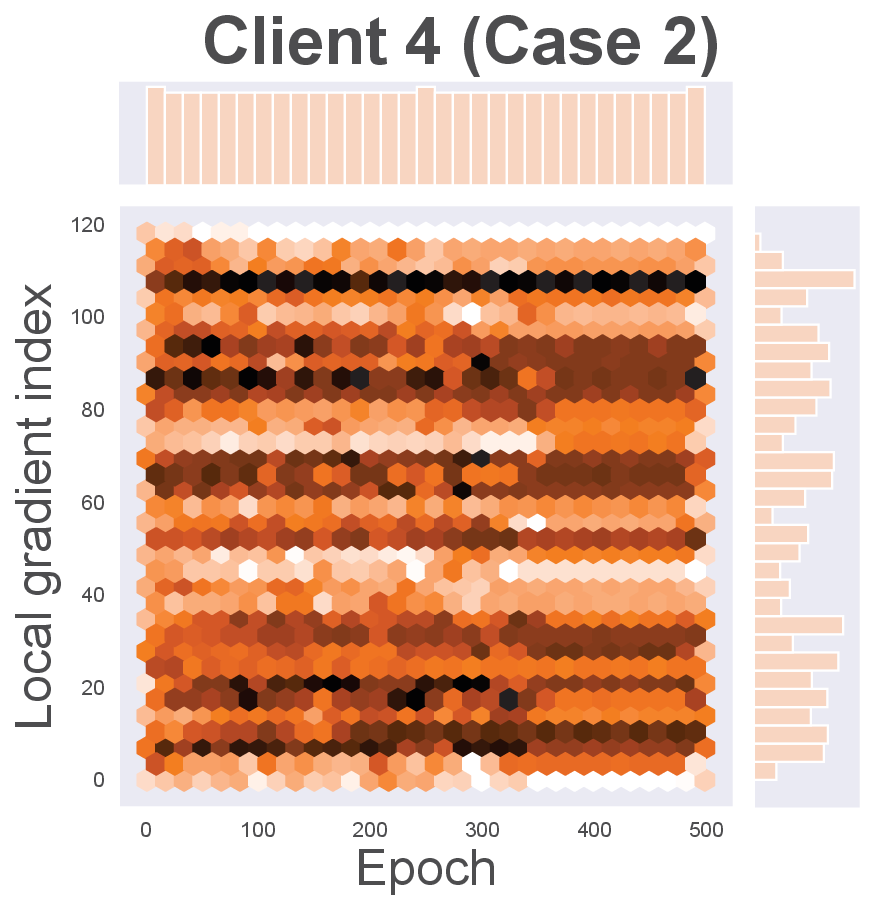}
\includegraphics[width=0.19\textwidth]{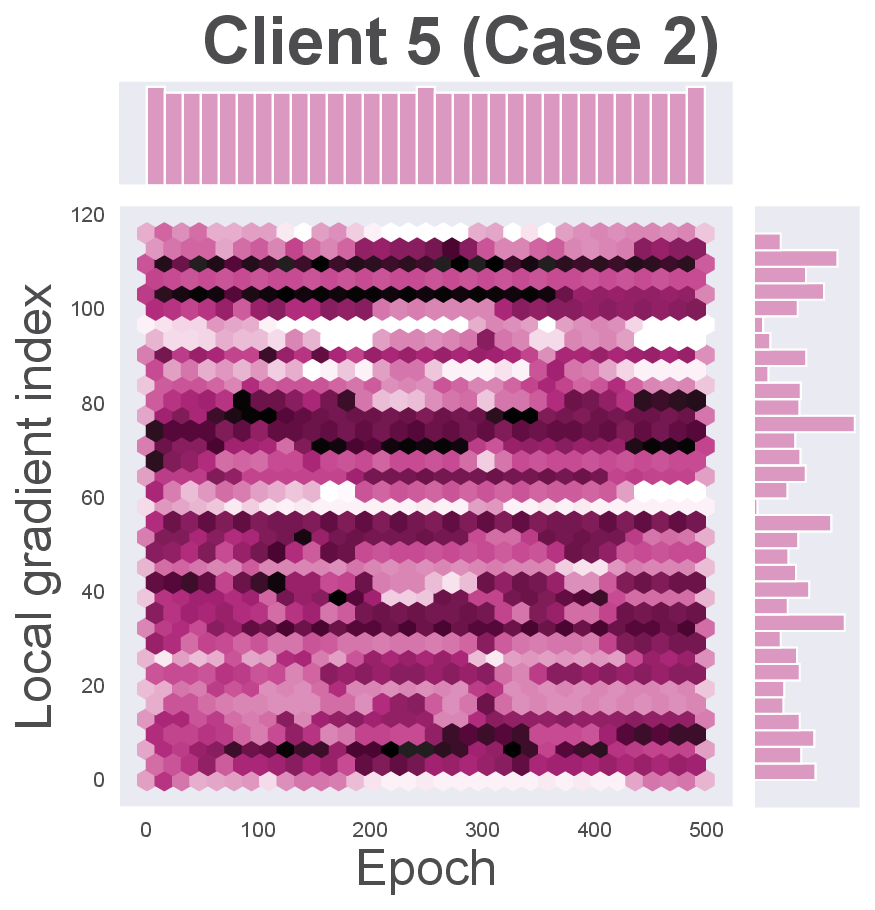}
\label{distribute_c2}}

\subfloat[]{\includegraphics[width=0.19\textwidth]{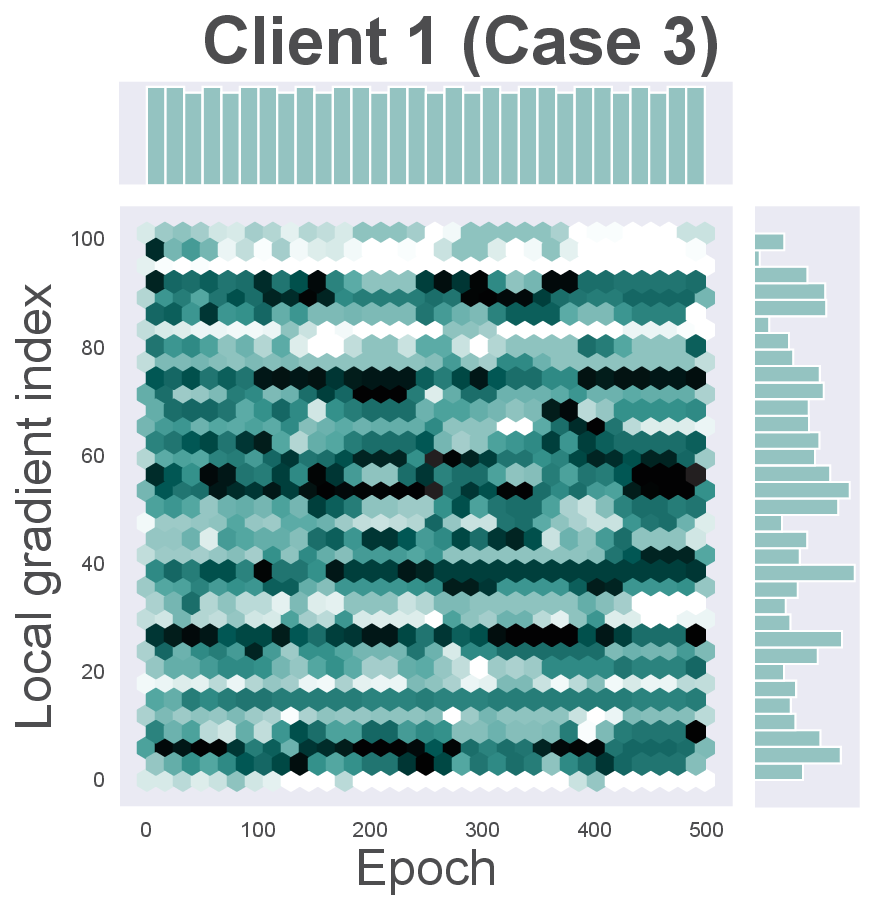}
\includegraphics[width=0.19\textwidth]{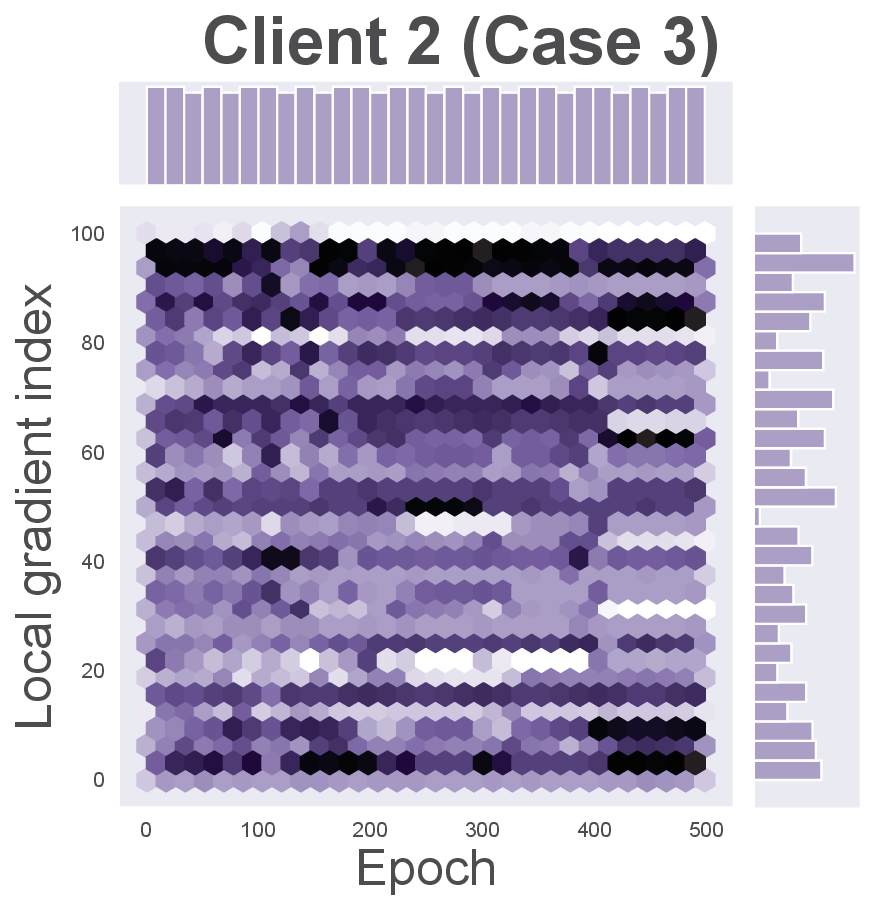}
\includegraphics[width=0.19\textwidth]{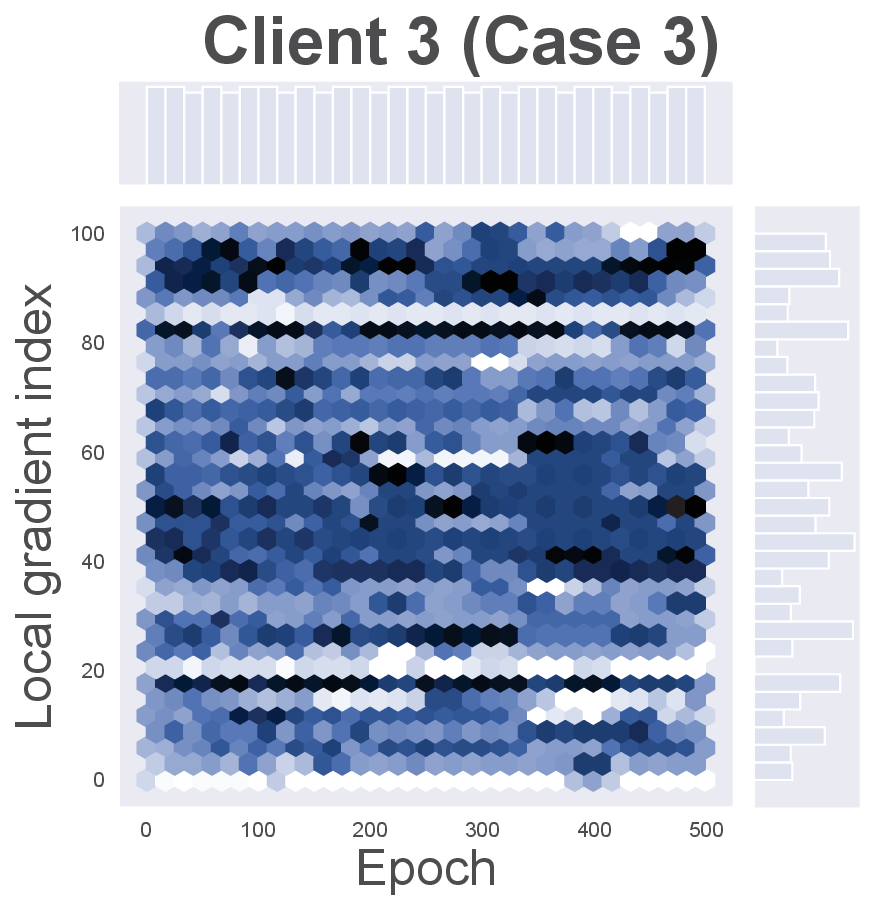}
\includegraphics[width=0.19\textwidth]{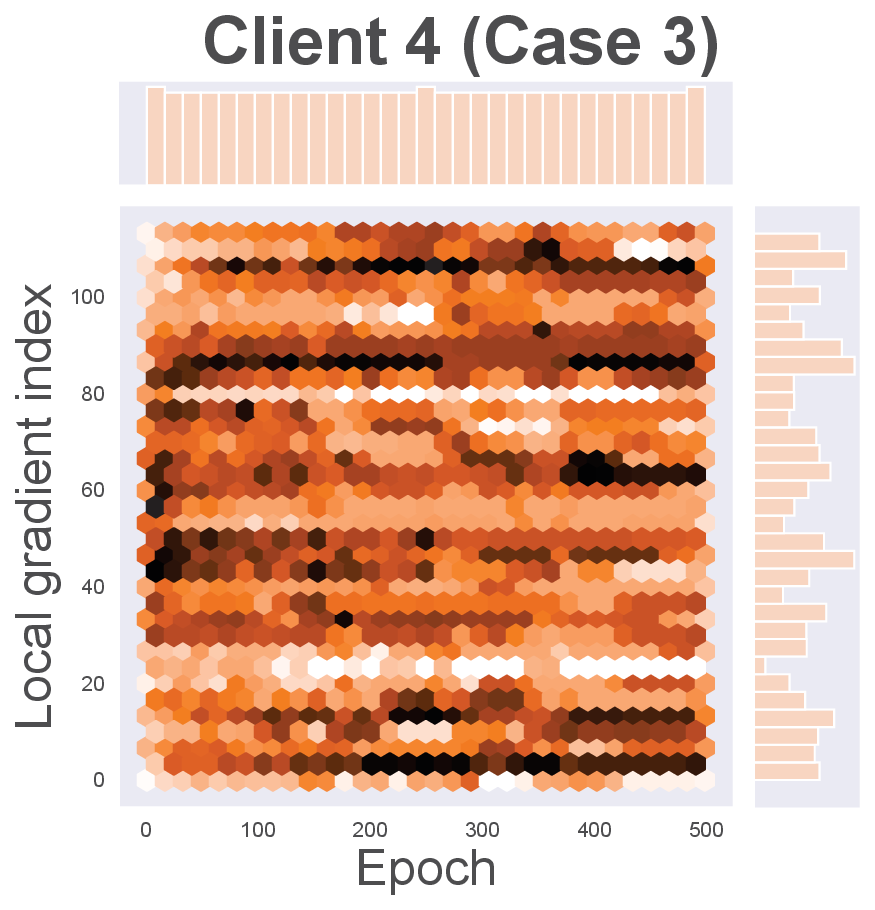}
\includegraphics[width=0.19\textwidth]{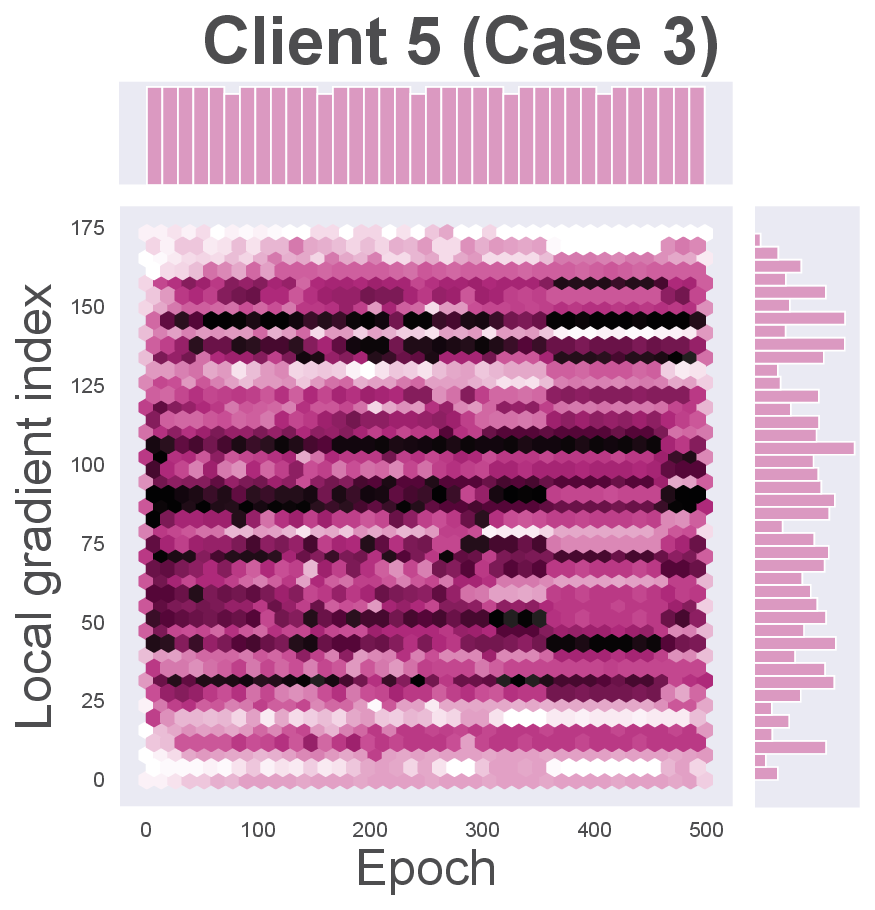}
\label{distribute_c3}}

\subfloat[]{\includegraphics[width=0.19\textwidth]{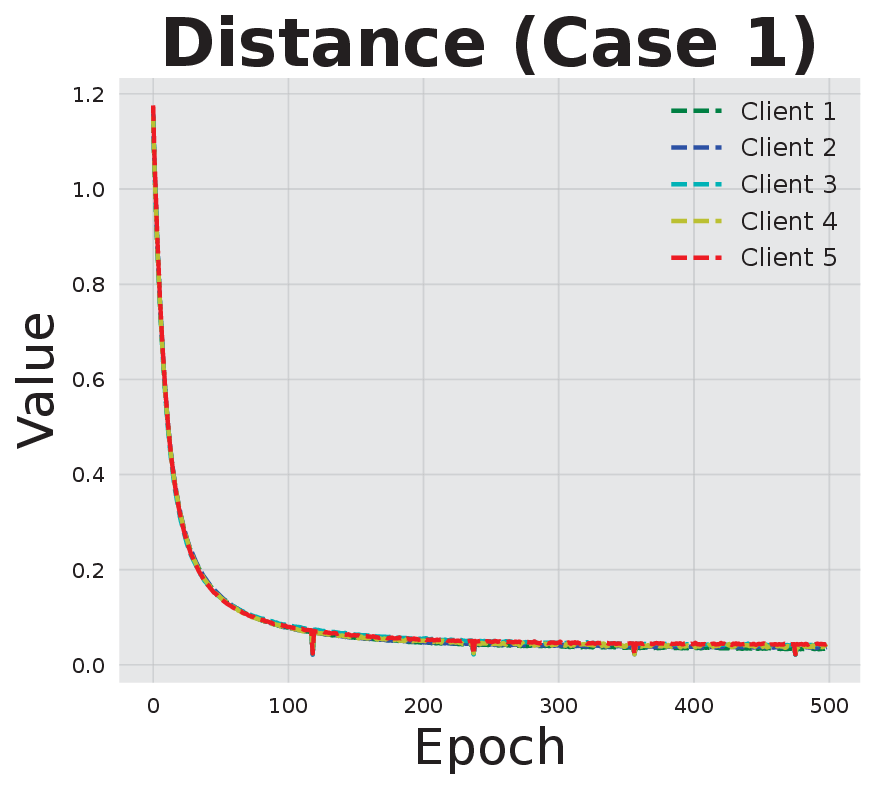}
\includegraphics[width=0.19\textwidth]{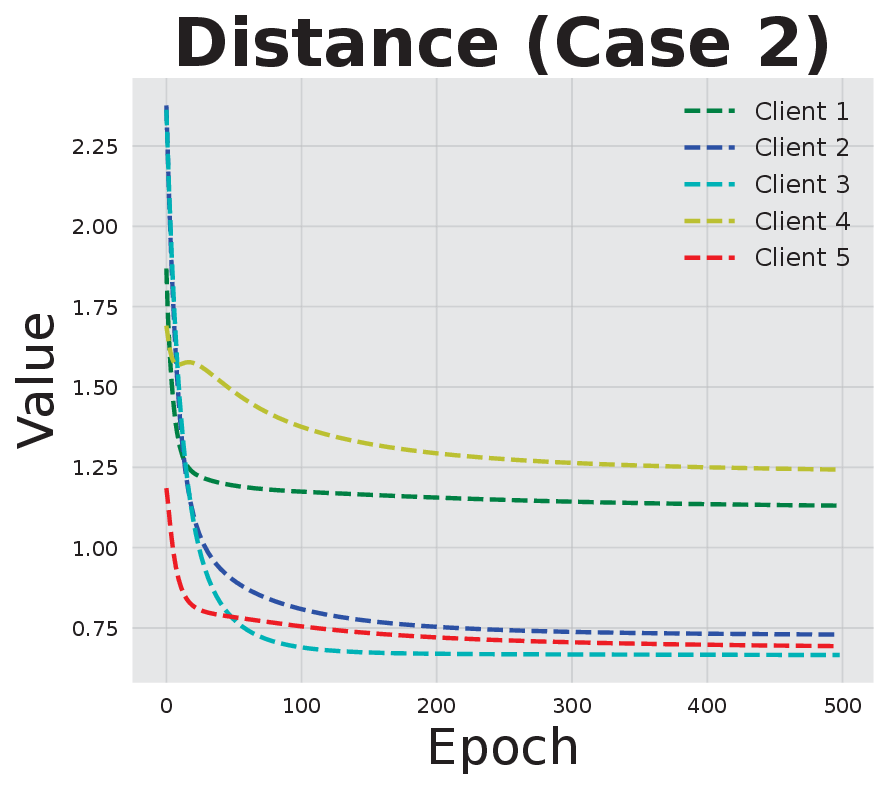}
\includegraphics[width=0.19\textwidth]{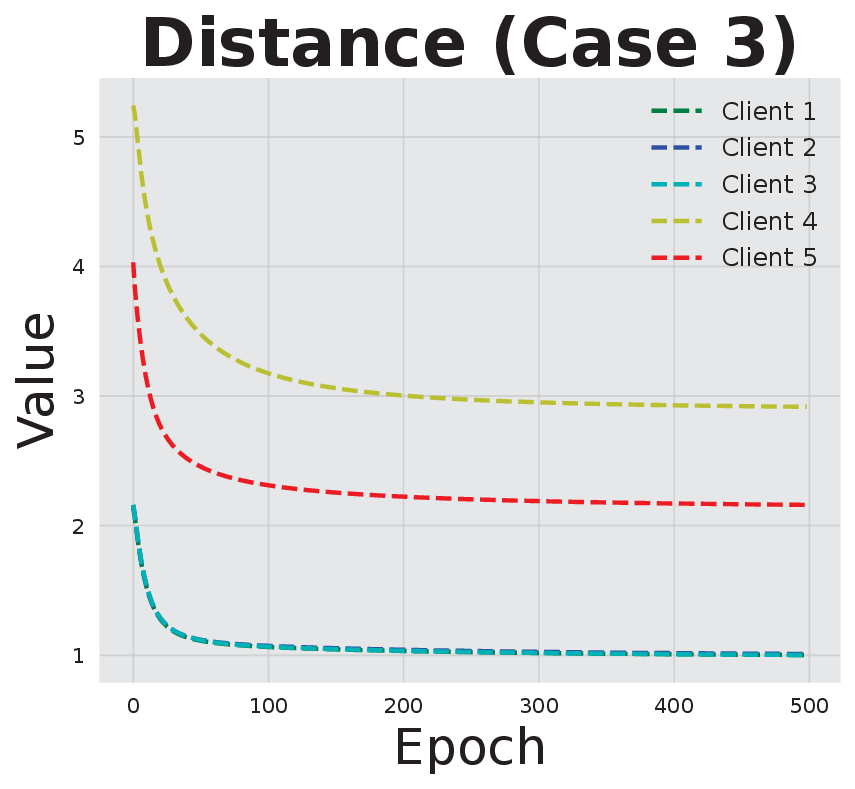}
\label{fig-distance}}
\hfil
\subfloat[]{\includegraphics[width=0.39\textwidth]{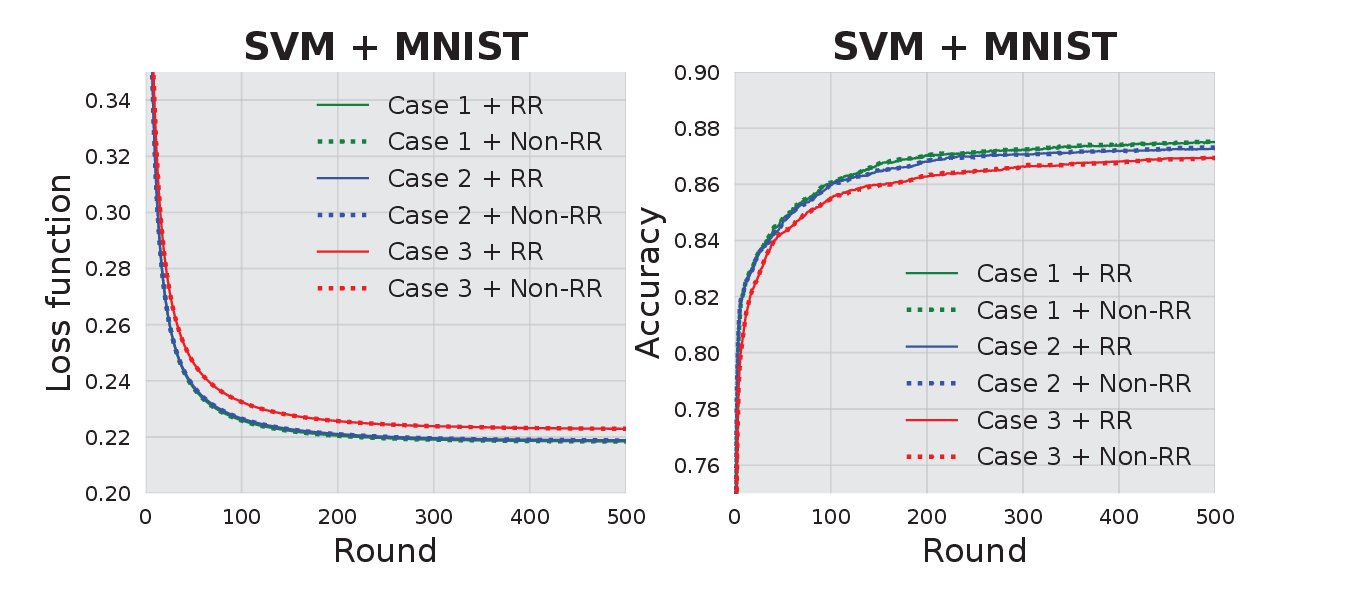}
\label{fig-RR}}

\caption{\footnotesize{
\textbf{Visualization demonstrating the validity of the gradients selected by BHerd.}
We continue to train the 'SVM+MNIST' model and dataset using the 'BHerd+FedAvg' framework with the hyperparameters previously specified; In order to more visually see the effect that the BHerd strategy produces on $\boldsymbol{g}_{(t,\sigma_i^\pi)}$, we show the results on five clients for \textbf{Case 1-3}.
(a), (b), (c) Visualization of the distribution of selected local gradients throughout the training process. We set the horizontal coordinate to be the number of epochs for which the client trains the local dataset, and the vertical coordinate to be the index of the selected local gradient; Since we set global round $T = 500$ and global epoch $E = 1$, the local epoch is equal to $ET=500$. The maximum value of the vertical coordinate is related to the number of local updates $\tau_{(t, i)}$ that we introduced in Section \ref{control}; The darker regions indicate clusters of selected local gradients.
(d) Visualization of the distance between the selected local gradient and the mean local gradient. We use  $\|\frac{1}{\alpha\tau_i}\boldsymbol{g}_{(t,\sigma_i^\pi)} - \boldsymbol{\mu}_{(t,\sigma_i)} \| $ to measure the distance between the gradient selected by our algorithm and the local average gradient; The legends for the individual curves are arranged from top to bottom as client 1-5.
(e) Random Reshuffling (RR) Protocol yields little enhancements.
The solid curves colored in green, blue, and red represent the outcomes obtained when employing the RR in \textbf{Case 1-3}, respectively, while the dashed lines denote the situations where the RR was not utilized.
}}
\label{fig-inner-logic}
\end{figure}

As shown in Figure \ref{distribute_c1}, \ref{distribute_c2} and \ref{distribute_c3}, most of the  $\boldsymbol{g}_{(t,\sigma_i^\pi)}$ on each client selected by BHerd strategy is concentrated in a fixed domain. This means that in the FL system, regardless of training IID or Non-IID local datasets, gradients with large deviations from the local average gradient $\boldsymbol{\mu}_{(t,\sigma_i)}$ always exist on each client in each round. Therefore, the BHerd strategy is to exclude these deviating gradients to select the remaining beneficial gradients, and this `benefit' is reflected in accelerating the convergence of the global model to the optimal. Meanwhile, it explains the meaning of $\alpha$: the proportion of beneficial gradients in the local gradient.

In Proposition \ref{proposition}, we replace $\boldsymbol{\mu}_{(t,\sigma_i)}$ with $\frac{1}{\alpha\tau_i}\boldsymbol{g}_{(t,\sigma_i^\pi)}$ due to the consistency of the local objective. Thus, in Figure \ref{fig-distance},  From the figure, it can be obtained that the distance values of all the clients in the corresponding \textbf{Case 1-3} are in a small range and decrease with the increase of epoch. For \textbf{Case 3}, client 1-3 are assigned half of the total dataset according to \textbf{Case 1} (IID) and thus have similar curves and the distance values remain minimal. While client 4-5 are assigned the other half of the total dataset according to \textbf{Case 2} (Non-IID) and thus have different curves and larger distance values. These experimental results show that the BHerd strategy coincides with the Proposition \ref{proposition} and is effective.

All the above analysis is based on the fact that the BHerd strategy does not use the RR protocol, so we analyze the impact of RR in the next Section.

\subsection{Impact of Random Reshuffling Protocol on System Performance}
\label{Random}
In order to validate our strategy of not using the RR protocol as mentioned in Section \ref{local update}, we do a comparative experiment on the RR protocol based on the experimental setup ($\alpha = 0.5$) in Section \ref{alpha}. The results of the loss function and classification accuracy are shown in Figure \ref{fig-RR}, we can see that in \textbf{Case 1-3}, whether or not using the RR protocol has little impact on the convergence curves and the final results. It is worth explaining that when using the RR protocol, the index of the selected gradient $\boldsymbol{g}_{(t,\sigma_i^\pi)}$ varies from epoch to epoch, making the distribution graph chaotic and detrimental to our analysis in Section \ref{distance}. Therefore, in the above experiments, we have used the Non-RR protocol to investigate the effect of other parameters in the FL system on the BHerd strategy.

\section{Discussion}
This paper proposes a BHerd strategy to improve the generalization performance of the global model by ordering and selecting the local gradients, so as to reduce the detrimental effect of Non-IID datasets on the convergence of the global model. We prove the convergence of BHerd theoretically and demonstrate its effectiveness and superiority through extensive experiments. Indeed, the BHerd strategy can be regarded as a contemporary approach to data \cite{sorscher2022beyond,qin2023infobatch,yang2022dataset} or gradient pruning \cite{lin2022fairgrape}, which has gained popularity in recent times. The same unanswered question that BHerd alongside these methods is: to what extent will the loss of sample information impact the final model? 

We can visualize in Figure \ref{fig-alpha} ($\alpha=0.1$) that when the beneficial sample information is insufficient to characterize the entire dataset, it can have disastrous consequences for the entire training process. In addition to this, although BHerd can select local gradients that are beneficial for global model convergence, the sample information discarded from this strategy has a greater impact on data-sensitive models (e.g., the CNN model in Figure \ref{lossacc:bherd-grab} and \ref{lossacc:nova-scafflod}). This effect is particularly evident in \textbf{Case 3}, where the difference between the local gradient and the global gradient is large due to the large degree of Non-IID in the data, and BHerd's measure of 'beneficial' is biased. Upon comparing the outcomes of CNN+CIFAR (\textbf{Case3}) in Figures \ref{lossacc:bherd-grab} and \ref{lossacc:nova-scafflod}, it is observed that BHerd-SCAFFLOD exhibits reduced oscillations relative to BHerd-FedAvg. As previously noted, SCAFFLOD introduces some gradient information from other clients into the local gradient of the current client thus reducing the bias between them. Consequently, future research will aim to delve into the underlying causes of this bias and identify effective strategies to ameliorate or potentially eradicate it.

An additional aspect meriting investigation pertains to the computational complexity and storage requirements of the BHerd strategy. Under the assumption that all clients execute $\tau$ epochs, the storage requisite for each local gradient is dictated by $d$ (the model dimension), while the computational overhead for a single comparison operation during sorting is denoted by $\mathcal{O}(1)$. When using the BHerd strategy, the computational complexity will increase from $\mathcal{O}(1)$ to $\mathcal{O}(\tau !)$ and the storage space will increase from $d$ to $\tau d$. While employing the GraB strategy for local gradient sorting does not escalate time complexity or storage requirements, observations from Figure \ref{lossacc:bherd-grab} indicate that the GraB-FedAvg algorithm yields marginal improvements in model convergence. In subsequent research, we will delve into the underlying causes of the shortcomings observed in the GraB strategy and endeavor to identify an efficacious ranking methodology that aligns with the principles of GraB.

It is pertinent to note, as demonstrated in Figure \ref{fig-parameter}, that our experimental investigation into the influence of individually set hyperparameters on training outcomes reveals varying degrees of impact. However, this study does not delve into the interplay among these hyperparameters. An avenue for future exploration might involve adaptively adjusting these hyperparameters in each training round to potentially expedite model convergence. 

\section{Methods}
\label{sec-preliminaries}
\subsection{Collect Local Gradients}
\label{local update} 
We assign weights to individual clients in the traditional approach (FedAvg \cite{mcmahan2017communication}),

\begin{equation}
	p_i{:}=\frac{|\mathcal{D}_i|}{|\mathcal{D}|}
\end{equation}
where $ \left| . \right| $ denotes the size of the set and $|\mathcal{D}|=\sum_{j=1}^N|\mathcal{D}_j|$ is the size of total training dataset. In the FL update rules, the training process has $ T $ ($ T \geq 1 $) communication rounds and the server has global model parameters $ \boldsymbol{w}_{t} $, where $ t = 0,1,2, ..., T $ denotes the round index. In each round, each client $ i $ training local dataset for $E$ epochs, thus has $ \tau_{i}= E\frac{|\mathcal{D}_i|}{B}$ ($ \tau_{i} \geq 1 $) local SGD iterations (number of batches)  and its local model parameters $ \boldsymbol{w}_{(t,i)}^\lambda $, where $ \lambda = 0,1,2,..., \tau_i $ denotes the local SGD iteration index. 

When FL begins ($ t=0 $), the server initializes the global model parameters $ \boldsymbol{w}_{1} $ and sends it to all clients. At $ \lambda=1 $, the local model parameters for all clients are received from the server that $\boldsymbol{w}_{(t,i)}^{\lambda=1}=\boldsymbol{w}_t$. For $ 1 \leq \lambda \leq \tau_i $, the gradients $ \nabla F_i(\boldsymbol{w}_{(t,i)}^\lambda;\boldsymbol{x}_{\sigma_i}) $ are computed according to the FedAvg local updates that
\begin{equation}
	\label{local-update}
	\boldsymbol{w}_{(t,i)}^{\lambda+1}=\boldsymbol{w}_{(t,i)}^\lambda-\eta\nabla F_i(\boldsymbol{w}_{(t,i)}^\lambda;\boldsymbol{x}_{\sigma_i}),
\end{equation}
where $\eta$ is learning rate, $ F_{i}\left(\boldsymbol{w}\right) $ is the local loss function,  $ \boldsymbol{w}_{(t,i)}^{\lambda} $ is the local model parameters and $\boldsymbol{x}_{\sigma_i}$ is the local data batch. In BHerd strategy, $\sigma_i=\{1,2,...,\tau_i\}$ denotes a permutation (ordering) of data batch $\boldsymbol{x}$, and it does not vary from epoch to epoch as if it were Random Reshuffling (RR) \cite{mishchenko2020random}. We have used experimental results in Section \ref{Random} to show that there is little difference between using and not using the RR protocol in herding strategies, and that not using the RR protocol makes it more convenient for us to analyze our results.

Next we will sort these gradients according to the idea of herding algorithm \cite{lu2022grab}.

\subsection{Herding Local Gradients}
\label{herding-local}
The herding problem originates from Welling \cite{10.1145/1553374.1553517} for sampling from a Markov random field that agrees with a given data set. Its discrete version is later formulated in Harvey and Samadi \cite{harvey2014near}. Concretely, Lu et al. \cite{lu2022general} proved that for any model parameters $\boldsymbol{x} \in \mathbb{R}^{d}$, if sums of consecutive stochastic gradients converge faster to the full gradient, then the optimizer will converge faster. Formally, in one epoch at client $i$, given $\tau_{i}$ gradients $\left\lbrace\boldsymbol{z}_{\lambda}\right\rbrace _{\lambda=1}^{\tau_i}\in \mathbb{R}^{d} $, any permutation  $\sigma_{i}^{\pi}=\{\pi_1,\pi_2,...,\pi_{\tau_{i}}\}$ minimizing the term (named average gradient error)
\begin{equation}
	\label{formulate-herding}
	\max_{k\in\{1,\ldots,\tau_i\}}\left\|\sum_{\mu=1}^{k}\left(\nabla \boldsymbol{z}_{\pi_\mu} -\frac{1}{\tau_{i}} \sum_{\lambda=1}^{\tau_i}\boldsymbol{z}_{\lambda}\right)\right\|,
\end{equation}
leads to fast convergence. For ease of representation in BHerd strategy, we define 
\begin{equation}
	\label{mu}
	\boldsymbol{\mu}_{(t,\sigma_i)}:=\frac1{\tau_i}\sum_{\lambda=1}^{\tau_i}\nabla F_i(\boldsymbol{w}_{(t,i)}^\lambda;\boldsymbol{x}_{\sigma_i})
\end{equation}
to denote the average of the gradients which we obtained in Section \ref{local update}. 

We sort the gradients collected to be close to the average gradient $\boldsymbol{\mu}_{(t,\sigma_i)}$ and pick only the top $\alpha\tau_{i}, \alpha\in\left( 0,1\right] $ (rounding when not an integer) gradients to form a new permutation. Based on these, we rewrite the Equation \eqref{local-update} as Equation \eqref{local-obj}. The purpose of Equation \eqref{local-obj} is to characterize the entire set of gradients in terms of partial of them, thus mitigating the effects of long-tailed gradients and improving the generalization ability of the model. In Appendix \ref{comparison} and \ref{sec-analysis}, we will theoretically analyze and experimentally validate this strategy.

\subsection{Global Aggregation}
\label{global-agg}
After herding and selecting the collected local gradients, we sum the gradients in the permutation $\sigma_i^{\pi}$ to get
\begin{equation}
	\label{sum-herd}
	\boldsymbol{g}_{(t,\sigma_i^\pi)}=\sum_{\lambda=1}^{\alpha\tau_i}\nabla F_i{\left(\boldsymbol{w}_{(t,i)}^\lambda;\boldsymbol{x}_{\sigma_i^\pi}\right)}.
\end{equation}
In the end of $t$-th global round, when all clients have completed this step, $\boldsymbol{g}_{(t,\sigma_i^\pi)}$ is sent to the parameter server side to complete the aggregation update step as 
\begin{equation}
	\label{global-aggregation}
	\begin{aligned}
		\boldsymbol{w}_{t+1}&=\boldsymbol{w}_{t}-\eta\frac{E}{\alpha}\sum_{i=1}^{N}p_{i}\boldsymbol{g}_{(t,\sigma_{i}^{\pi})} \\
		&=\boldsymbol{w}_t-\frac\eta\alpha\sum_{i=1}^Np_i\boldsymbol{g}_{(t,\sigma_i^\pi)}.
	\end{aligned}
\end{equation}
The last item is due to the fact that we set $E=1$ which is reasonable, i.e., we only consider the case where each client is trained for only 1 epoch, since too large or too small an epoch can lead to a bias in the global model and thus affect convergence.

\section*{Acknowledgements}
Acknowledgements are not compulsory. Where included they should be brief. Grant or contribution numbers may be acknowledged.

Please refer to Journal-level guidance for any specific requirements.


\bibliography{sn-bibliography}

\begin{appendices}
\section{BHerd-FedAvg vs FedAvg}
\label{comparison}
For the purpose of comparison with BHerd's new herding strategy, we will introduce here the most primitive FL algorithm: the FedAvg. The logic of FedAvg with distributed SGD is presented in Algorithm \ref{fedavg-alg}, which ignores aspects related to the communication between the server and clients \cite{mcmahan2017communication}. After the previous description of the training process of FL, the FedAvg algorithm becomes clear, and can be formally described as
\begin{equation}
	\label{fedavg}
	\boldsymbol{w}_{t+1}=\boldsymbol{w}_{t}-\eta\sum_{i=1}^{N}p_{i}\sum_{\lambda=1}^{\tau_i}\nabla F_i(\boldsymbol{w}_{(t,i)}^\lambda;\boldsymbol{x}_{\sigma_i}),
\end{equation}
where it can be seen as a special case of our herding strategy in the case of $E=1$ and $\alpha=1$. According the definition of Equation \eqref{mu}, the Equation \eqref{fedavg} can also be represented as $\boldsymbol{w}_{t+1}=\boldsymbol{w}_{t}-\eta\sum_{i=1}^{N}p_{i}\tau_i\boldsymbol{\mu}_{(t,\sigma_i)}$. Compare to FedAvg, a natural research question is

\textit{How does our new herding strategy works and is it more effective than FedAvg?}
\begin{proposition}
	\label{proposition}
	According to our definition of global aggregation for the new herding strategy, its gradient aggregation is equivalent to model parameter aggregation, that is
	\begin{equation}
		\label{parameter-agg}
		\boldsymbol{w}_{t+1}=\sum_{i=1}^Np_i\boldsymbol{w}_{(t,i)}^{\tau_i+1}
	\end{equation}
\end{proposition}
\begin{proof}
	For details, see Appendix \ref{appendix1}.
\end{proof}
Equation \eqref{parameter-agg} is another representation of FedAvg's global aggregation, which performs global model parameter updating by aggregating the local model parameters obtained after local updating. Thus we can answer the first half of the question posed: BHerd strategy works in line with FedAvg and is an extension of it. 

For the second half of the problem, we will give the convergence analysis in Section 4 and the experimental validity analysis in Section 5.In this section of the paper, we first give the abstract analysis in a rough way.

We first simplify all the operations of BHerd strategy to the two-dimensional plane, i.e., model parameters $\boldsymbol{w} \in \mathbb{R}^{2}$, the local gradient $\nabla F_i(\boldsymbol{w}) \in \mathbb{R}^{2}$ and the global gradient $\nabla F(\boldsymbol{w})\in \mathbb{R}^{2}$. Then make an abstraction of their variations according to the training process we defined above. We assume that there are two clients, one with 6 local update iterations ($\tau_{i}=6$) and the other with 4 local update iterations ($\tau_i=4$). A permutation of the local gradient ${\nabla F_i(\boldsymbol{w}_{(t,i)}^\lambda;\boldsymbol{x}_{\sigma_i})}$ is obtained on each client $i$ in the order of $\sigma_i $. It is then re-ordered using the herding policy and the top $\alpha\tau_i$ elements are selected. Up to this point, we can intuitively see that BHerd herding strategy is able to characterize all the gradient information with a portion of the local gradient. This has the advantage of reducing the impact of the biased gradient ${\nabla F_i(\boldsymbol{w}_{(t,i)}^\lambda;\boldsymbol{x}_{\sigma_i-\sigma_i^\pi})}$, which allows us to increase the step size ($\frac{1}{\alpha}$) global gradient to accelerate the convergence. This conclusion can also be easily generalized to $\mathbb{R}^{d}$-space. 

In the next section, we will apply the proposed herding strategy and the new FL training procedure to design the algorithm and formally analyze its convergence.
\section{BHerd strategy}
\label{sec-analysis}
In this section, based on the definitions and FL update rule of the new herding strategy obtained in Section \ref{sec-preliminaries}, we first designed the ordering algorithm and the FL update algorithm. Next, we analyze the convergence of the BHerd strategy and get a convergence upper bound in the non-convex hypothesis.

\subsection{FL with Herding}
\label{algorithm}
The greedy ordering algorithm for herding can be used to minimize Equation \eqref{local-obj} that find a better data permutation $ \sigma_i^{\pi} $, and is shown in Algorithm \ref{herding-order} which starts with a group of gradients $\{\boldsymbol{z}_{\lambda}\}_{\lambda=1}^{\tau}$. After finding the average gradient $\frac{1}{\tau}\sum_{\mu=1}^{\tau}\boldsymbol{z}_{\mu}$, all the gradients subtract this average gradient to get the deviations, which are assigned as the new gradients to $\{\boldsymbol{z}_{\lambda}\}_{\lambda=1}^{\tau}$ (Line 1 of Algorithm \ref{herding-order}). 

To obtain the new order $\sigma_i^\pi{:}\{\pi_1,\pi_2,\ldots,\pi_{\alpha\tau_i}\}$, Algorithm \ref{herding-order} needs to select $\alpha\tau_i$ gradients from the set $\{\boldsymbol{z}_{\lambda}\}_{\lambda=1}^{\tau}$ without repetition, aiming to minimize the value of $\left\|\boldsymbol{z}_{\pi_1}+\cdots+\boldsymbol{z}_{\pi_{\alpha\tau_i}}\right\|$. The specific approach is as follows: at each iteration, a single gradient $\boldsymbol{z}_{\mu}$ is chosen to minimize the $\|\boldsymbol{s}+\boldsymbol{z}_{\mu}\|$ value, where $\boldsymbol{s}$ represents the cumulative of the selected gradients (Line 4 and 5 of Algorithm \ref{herding-order}). After selecting a gradient $\boldsymbol{z}_{\mu}$, it is removed from the gradient set $\{\boldsymbol{z}_{\lambda}\}_{\lambda=1}^{\tau}$, and the order $\pi$ is reordered according to the selection index $\mu$ Line 5 of Algorithm \ref{herding-order}.

Based on the herding strategy in Algorithm \ref{herding-order}, we can clearly get the framework of FL under the new update rule: when training the local model for an epoch, we herd all the local gradients in the current global round and take the first a items of them to form a new permutation, and we send all the local gradients in this new permutation to the parameter server for global update. This herding-gradient approach is formally described in Algorithm \ref{fed-herding}.

Algorithm \ref{fed-herding} can be considered as adding the herding operation (Algorithm \ref{herding-order}) of the local gradients to the FL framework (Algorithm \ref{fedavg-alg}), where the local gradient $ \nabla F_i(\boldsymbol{w}_{(t,i)}^\lambda;\boldsymbol{x}_{\sigma_i}) $ corresponds to the vector $\{\boldsymbol{z}_{\lambda}\}_{\lambda=1}^{\tau}$ in Algorithm \ref{herding-order}. After obtaining a new permutation $ \sigma_i^{\pi} $, we can compute the corresponding $\boldsymbol{g}_{(t,\sigma_i^\pi)}$ by using Equation \eqref{sum-herd}, and finally use the new aggregation rule Equation \eqref{global-aggregation} to update the global model (Line 7, 8 and 10 in Algorithm \ref{fed-herding}).
\subsection{Convergence Analysis}
In FL, for each global model $ \boldsymbol{w} $ there is its corresponding loss function $ F(\boldsymbol{w}) $. The objective of learning is to find optimal global model $ \boldsymbol{w}^{*} $ to minimize $ F(\boldsymbol{w}) $, which is formally described as
\begin{equation}
	\boldsymbol{w^{*}} = \arg\min F(\boldsymbol{w})
\end{equation}
For the global loss function  $ \boldsymbol{w} $, it is also necessary to define its global gradient. The following will describe how we obtain this global gradient in FL.
\begin{definition}
	\label{def}
	At each client $i$ in round $t$, we train the global model parameter $\boldsymbol{w}_{t}$ with the local datasets, with the difference that this model parameter will not be involved in the local update of Equation \eqref{local-update}, but will remain fixed at round $t$ and across all clients.
	$$\nabla F_i(\boldsymbol{w}_t):=\frac1{\tau_i}\sum_{\lambda=1}^{\tau_i}\nabla F_i(\boldsymbol{w}_t;\boldsymbol{x}_{\sigma_i})$$
	Then weighted sum $\nabla F_i(\boldsymbol{w}_t)$ at each client to obtain the global gradient $\nabla F(\boldsymbol{w}_t)$.
	$$\nabla F(\boldsymbol{w}_t):=\sum_{i=1}^Np_i\nabla F_i(\boldsymbol{w}_t)$$
\end{definition}

For further analysis, we use some assumptions for non-convex optimization.
\begin{assumption}[Lipschitz smooth.]
	\label{L-smooth}
	For any $w,v\in\mathbb{R}^{d}$, it holds that
	$$\|\nabla F(\boldsymbol{w})-\nabla F(\boldsymbol{v})\|\leq L\|\boldsymbol{w}-\boldsymbol{v}\|$$
\end{assumption}

\begin{assumption}[Differences in model parameters]
\label{model diff}
    In ground $t\in[1,T]$, if we defined $\Delta_{t}:=\max_{i,\lambda}\|\boldsymbol{w}_{t}-\boldsymbol{w}_{(t,i)}^{\lambda}\| $ for all client $i\in[1,N]$ and local iteration $\lambda\in[2,\tau_i+1]$, it exist $\beta \in (0,1]$ that
    $$\beta\Delta_{t}\leq\|\boldsymbol{w}_{t}-\boldsymbol{w}_{t+1}\|$$
    
\end{assumption}

\begin{assumption}[Bounded gradient error.]
	\label{bound-error}
	For any $t\in[1,T]$ and $i\in[1,N]$, it holds that
	$$\frac{1}{\tau_{i}}\sum_{\lambda=1}^{\tau_{i}}\left\|\left(\nabla F_{i}(\boldsymbol{w}_{t};\boldsymbol{x}_{\sigma_{i}})-\nabla F_{i}(\boldsymbol{w}_{t})\right)\right\|^{2}\leq\varsigma^{2}$$
	
\end{assumption}

\begin{assumption}[PL condition.]
	\label{PL}
	We say the loss function $F$ fulfills the Polyak-Lojasiewicz (PL) condition if there exists $\gamma > 0$ such that for any $\boldsymbol{w}_t \in \mathbb{R}^{d}$,
	$$\frac12\|\nabla F(\boldsymbol{w})\|^2\geq\gamma\big(F(\boldsymbol{w})-F(\boldsymbol{w}^*)\big),$$
	where $F(\boldsymbol{w}^*)=\inf_{\boldsymbol{v}\in\mathbb{R}^d}F(\boldsymbol{v})$.
	
\end{assumption}
The symbol of $ \left\|.\right\| $ denotes the $L$ norm of vector, and we use the default $L_2$ norm in this paper. Assumptions \ref{L-smooth}, \ref{model diff} and \ref{bound-error} are commonly used assumptions in the study of SGD. Assumptions \ref{L-smooth} measure smooth upper bounds on the global and local loss functions. Assumption \ref{bound-error} expresses the maximum value of the difference between the gradient set $\nabla F_i(\boldsymbol{w}_t;\boldsymbol{x}_{\sigma_i})$ and its average gradient $\nabla F_i(\boldsymbol{w}_t)$ in Definition \ref{def}.

Before we analyze the upper bound on convergence, we introduce some of the following Lemmas.
\begin{lemma}
	\label{converge-bound}
	We denote the maximal number of local iterations for all clients by $\tau_{max}:=\max_{i\in[1,N]}\tau_{i}$. If $\alpha\eta L \tau_{max}\leq1$ with $ \alpha \in (0,1] $ holds and Assumption \ref{L-smooth} and \ref{model diff} hold, then
	$$\frac1T\sum_{t=1}^T\|\nabla F(\boldsymbol{w}_t)\|^2\leq\frac{2\left(F(\boldsymbol{w}_1)-F(\boldsymbol{w}^*)\right)}{\alpha\eta\tau_{max} T}+\frac{ L ^2}T\sum_{t=1}^T\Delta_t^2$$
\end{lemma}
\begin{proof}
	For details, see Appendix \ref{appendix2}. 
\end{proof}
In Lemma \ref{converge-bound}, with both the dataset and the model structure already determined, the loss value of its optimal model parameter $\boldsymbol{w}^*$ is uniquely determined. This means that $F(\boldsymbol{w}_1)-F(\boldsymbol{w}^*)$ is fixed and the magnitude of the value is determined by the initialized model parameter $\boldsymbol{w}_1$. Parameters $\alpha$, $\eta$, and $T$ are manually set fixed values. Parameter $\tau_{max}$ is a fixed value only with respect to the total dataset size $|\mathcal{D}|$ and the number of clients $N$. Parameter $ L $ is the smooth upper bound that we defined in Assumption \ref{L-smooth}, and is also fixed in the actual FL training.

Next, Lemma \ref{delta-bound} will prove the relationship between the upper bound of $\sum_{t=1}^T\Delta_t^2$ and the global gradient $\nabla F(\boldsymbol{w}_t)$.
\begin{lemma}
	\label{delta-bound}
		Under Assumptions \ref{L-smooth}, \ref{model diff} and \ref{bound-error}, if the learning rate $\eta$ fulfills $\eta\leq\frac{\sqrt{3}\beta}{6 L\tau_{max}}$ with $t\in[1,T-1]$, then the following inequalities hold:
	$$\Delta_{t+1}^{2}\leq\frac12{\beta^{2}}\eta^{2}\tau_{\max}^{2}\varsigma^{2}+\frac12{\beta^{2}}\eta^{2}\tau_{\max}^{2}\|\nabla F(\boldsymbol{w}_{t+1})\|^{2}$$
    and,
	$$\Delta_1^2\leq\frac6{\beta^2}\eta^2\tau_{\mathrm{max}}^2$$
    and finally,
	$$\sum_{t=1}^T\Delta_t^2\leq\frac{12}{\beta^{2}}(T-1)\eta^{2}\tau_{max}^{2}\left(\varsigma^{2}+\sum_{t=1}^{T-1}\parallel\nabla F(\boldsymbol{w}_{t})\parallel^{2}\right)+\frac{6}{\beta^{2}}\eta^{2}\tau_{max}^{2}\varsigma^{2}$$
	
\end{lemma}
\begin{proof}
	For details, see Appendix \ref{appendix2}. 
\end{proof}

According to the definition of $ \varsigma $ in Assumption \ref{bound-error}, we can always find a suitable fixed value of $ \varsigma $ in FL training. Then, it is only sufficient to carry over the result in Lemma \ref{delta-bound} to Lemma \ref{converge-bound}, which leads to Theorem \ref{theo-bound}.

\begin{theorem}
	\label{theo-bound}
	In the BHerd strategy, under Assumptions \ref{L-smooth}, \ref{model diff} and \ref{bound-error}, if the learning rate $\eta$ fulfills $\eta\leq\frac{\sqrt{3}\beta}{6 L\tau_{max}}$ and the total global round $T$ fulfills $ T\leq\frac{\beta^{2}}{24\eta^{2}L^{2}\tau_{max}^{2}} $, then it converges at the rate
	\begin{equation}
        \label{eq-bound}
	    \frac{1}{T}\sum_{t=1}^{T}\parallel\nabla F(\boldsymbol{w}_{t})\parallel^{2}\leq\frac{4(F(\boldsymbol{w}_{1})-F(\boldsymbol{w}^{*}))}{\alpha\eta\tau_{max}T}+\frac{\varsigma^{2}}{T}+\frac{\varsigma^{2}}{2T^{2}}
	\end{equation}
\end{theorem}

\begin{proof}
	For details, see Appendix \ref{appendix2}. 
\end{proof}
From Theorem \ref{theo-bound}, the right-hand side of the inequality in Equation \ref{eq-bound} is a fixed value when the appropriate values of $T$ and $\eta$ are selected, proving that our algorithm converges below a definite value.

\begin{theorem}
\label{PL-bound}
   Under Assumption \ref{L-smooth}, \ref{model diff}, \ref{bound-error} and \ref{PL} (PL condition), if $\eta\leq\frac{\sqrt{3}\beta}{6 L\tau_{max}}$ and $ T\leq\frac{\beta^{2}}{24\alpha\eta^{3}L^{3}\tau_{max}^{3}} $ hold, then the BHerd strategy converges at the upper bound
		$$F(\boldsymbol{w}_T)-F(\boldsymbol{w}^*)=\mathcal{O}(\frac{\varsigma^{2}}{T\eta L\tau_{max}\gamma})+\mathcal{O}(\frac{F}{T\eta L\tau_{max}})+\mathcal{O}(\frac{\alpha\varsigma^{2}}{TL}), $$
		where $\mathcal{O}$ swallows all other constants and $F:=F(\boldsymbol{w}_1)-F(\boldsymbol{w}^*)$ .
\end{theorem}
\begin{proof}
	For details, see Appendix \ref{appendix2}. 
\end{proof}

\setcounter{theorem}{0}
\setcounter{proposition}{0}
\setcounter{lemma}{0}

\section{Partial Gradients vs. Total Gradients}
\label{appendix1}
\begin{proposition}
	According to our definition of global aggregation for the new herding strategy, its gradient aggregation is equivalent to model parameter aggregation, that is
	
	$$\boldsymbol{w}_{t+1}=\sum_{i=1}^Np_i\boldsymbol{w}_{(t,i)}^{\tau_i+1}$$
\end{proposition}
\begin{proof}
	In order to proceed to the proof, we first rewrite the local objective
	
	\begin{equation}
		\label{re-local-obj}
		\begin{aligned}
			\sigma_i^{\pi}&=\arg \min\left\|\sum_{\lambda=1}^{\alpha\boldsymbol{\tau}_i}\left(\nabla F_i\left(\boldsymbol{w}_{(t,i)}^{\lambda};\boldsymbol{x}_{\sigma_i^{\pi}}\right)-\boldsymbol{\mu}_{(t,\sigma_i)}\right)\right\| \\
			&=\arg \min\left\|\frac{1}{\alpha\tau_i}\sum_{\lambda=1}^{\alpha\tau_i}\left(\nabla F_i\left(\boldsymbol{w}_{(t,i)}^\lambda;x_{\sigma_i^T}\right)-\boldsymbol{\mu}_{(t,\sigma_i)}\right)\right\|^2 \\
			&=\arg\min\left\|\frac1{\alpha\tau_i}\boldsymbol{g}_{(t,\sigma_i^\pi)}-\boldsymbol{\mu}_{(t,\sigma_i)}\right\|^2 \\
		\end{aligned}
	\end{equation}
	
	It is easy to see that our goal is to select a small number of gradients $\frac1{\alpha\tau_i}\boldsymbol{g}_{(t,\sigma_i^\pi)}$  to characterize the average of all gradients $\boldsymbol{\mu}_{(t,\sigma_i)}$ at each client, and we note that $\boldsymbol{w}_{(t,i)}^{\lambda=1}=\boldsymbol{w}_t$. Thus, just rewrite the Equation \eqref{global-aggregation}, we get 
	
	$$\begin{aligned}
		&\boldsymbol{w}_{t+1}=\boldsymbol{w}_{t}-\frac{\eta}{\alpha}\sum_{i=1}^{N}p_{i}\boldsymbol{g}_{(t,\sigma_{i}^{\pi})} \\
		&=\boldsymbol{w}_t-\eta\sum_{i=1}^Np_i\frac1\alpha\boldsymbol{g}_{(t,\sigma_i^n)} \\
		&=\sum_{i=1}^Np_i\left(\boldsymbol{w}_{(t,i)}^{\lambda=1}-\eta\tau_i\boldsymbol{\mu}_{(t,\sigma_i)}\right) \\
		&=\sum_{i=1}^Np_i\left(\boldsymbol{w}_{(t,i)}^{\lambda=1}-\eta\sum_{\lambda=1}^{\tau_i}\nabla F_i(\boldsymbol{w}_{(t,i)}^\lambda;\boldsymbol{x}_{\sigma_i})\right) \\
		&=\sum_{i=1}^Np_i\boldsymbol{w}_{(t,i)}^{\tau_i+1}
	\end{aligned}$$
	
	That completes the proof.
\end{proof}
\section{Convergence Analysis}
\label{appendix2}
Now we start proving some lemmas.

\begin{lemma}
	We denote the maximal number of local iterations for all clients by $\tau_{max}:=\max_{i\in[1,N]}\tau_{i}$. If $\alpha\eta L \tau_{max}\leq1$ with $ \alpha \in (0,1] $ holds and Assumption \ref{L-smooth} and \ref{model diff} hold, then
	$$\frac1T\sum_{t=1}^T\|\nabla F(\boldsymbol{w}_t)\|^2\leq\frac{2\left(F(\boldsymbol{w}_1)-F(\boldsymbol{w}^*)\right)}{\alpha\eta\tau_{max} T}+\frac{ L ^2}T\sum_{t=1}^T\Delta_t^2$$
\end{lemma}

\begin{proof}
	By the Taylor Theorem, for all the $t\in[1,T]$,
	
	$$\begin{aligned}
		&F(\boldsymbol{w}_{t+1})\leq F(\boldsymbol{w}_{t})-\langle\nabla F(\boldsymbol{w}_{t}),\eta\sum_{i=1}^{N}p_{i}\boldsymbol{g}_{(t,\sigma_{i}^{\pi})}\rangle+\frac{L}{2}\left\|\eta\sum_{i=1}^{N}p_{i}\boldsymbol{g}_{\sigma_{i}^{\pi}}\right\|^{2} \\
		&\leq F(\boldsymbol{w}_t)-\alpha\eta\tau_{max}\left<\nabla F(\boldsymbol{w}_t),\sum_{i=1}^Np_i\frac1{\alpha\tau_i}\boldsymbol{g}_{(t,\sigma_i^\pi)}\right> +\frac{\alpha^2\tau_{max}^2\eta^2L}2\left\Vert\sum_{i=1}^Np_i\frac1{\alpha\tau_i}\boldsymbol{g}_{(t,\sigma_i^\pi)}\right\Vert^2 \\
		&=\left.F(\boldsymbol{w}_t)-\frac{\alpha\eta\tau_{max}}2\|\nabla F(\boldsymbol{w}_t)\|^2-\frac{\alpha\eta\tau_{max}}2\left\|\sum_{i=1}^Np_i\frac1{\alpha\tau_i}\boldsymbol{g}_{(t,\sigma_i^\pi)}\right\|^2\right.  \\
		&+\frac{\alpha\eta\tau_{max}}{2}\left\|\nabla F(\boldsymbol{w}_{t})-\sum_{i=1}^{N}p_{i}\frac{1}{\alpha\tau_{i}}\boldsymbol{g}_{(t,\sigma_{i}^{\pi})}\right\|^{2}
		+\frac{\alpha^2\tau_{max}^2\eta^2L}2\left\|\sum_{i=1}^Np_i\frac1{\alpha\tau_i}\boldsymbol{g}_{(t,\sigma_i^\pi)}\right\|^2 \\
		& \leq F(\boldsymbol{w}_t)-\frac{\alpha\eta\tau_{max}}2\|\nabla F(\boldsymbol{w}_t)\|^2+\frac{\alpha\eta\tau_{max}}2\left\|\nabla F(\boldsymbol{w}_t)-\sum_{i=1}^Np_i\frac1{\alpha\tau_i}\boldsymbol{g}_{(t,\sigma_i^\pi)}\right\|^2
	\end{aligned}$$
	
	In the second step, we use the fact that $\tau_{max}\leq\tau_i$ for any client $i$. In the third step, we apply $-\langle a,b\rangle=-\frac{1}{2}\|a\|^{2}-\frac{1}{2}\|b\|^{2}+\frac{1}{2}\|a-b\|^{2}$. In the last step, we use the condition that $\alpha\eta L \tau_{max}<1$. For the last term, using the Jensen Inequality, we get
	
	$$\begin{aligned}
		\left\|\nabla F(\boldsymbol{w}_t)-\sum_{i=1}^Np_i\frac1{\alpha\tau_i}\boldsymbol{g}_{(t,\sigma_i^\pi)}\right\|^2
		&=\left\|\sum_{i=1}^Np_i\nabla F_i(\boldsymbol{w}_t)-\sum_{i=1}^Np_i\frac1{\alpha\tau_i}\boldsymbol{g}_{(t,\sigma_i^\pi)}\right\|^2 \\
		&\leq\sum_{i=1}^Np_i\left\|(\nabla F_i(\boldsymbol{w}_t)-\frac1{\alpha\tau_i}\boldsymbol{g}_{(t,\sigma_i^n)})\right\|^2
	\end{aligned}$$
	
	According to the Equation\eqref{re-local-obj}, we replace $\nabla F_i(\boldsymbol{w}_t)-\boldsymbol{\mu}_{(t,\sigma_i)}$ with $\nabla F_i(\boldsymbol{w}_t)-\frac1{\alpha\tau_i}\boldsymbol{g}_{(t,\sigma_i^\pi)}$. Then, using Assumption \ref{bound-error}, we have
	
	$$\begin{aligned}
		\left\|\nabla F(\boldsymbol{w}_{t})-\sum_{i=1}^{N}p_{i}\frac{1}{\alpha\tau_{i}}\boldsymbol{g}_{(t,\sigma_{i}^{\pi})}\right\|^{2}& \leq\sum_{i=1}^Np_i\left\|\nabla F_i(\boldsymbol{w}_t)-\boldsymbol{\mu}_{(t,\sigma_i)}\right\|^2  \\
		&\leq L ^2\sum_{i=1}^Np_i\left\|\boldsymbol{w}_t-\boldsymbol{w}_{(t,i)}^{\lambda}\right\|^2 \\
		&\leq L ^2\sum_{i=1}^Np_i\Delta_{t}^{2} \\
		&= L ^{2}\Delta_{t}^{2}
	\end{aligned}$$
	
	Put it back, we obtain
	\begin{equation}
		\label{inter-formula}
		F(\boldsymbol{w}_{t+1})\leq F(\boldsymbol{w}_t)-\frac{\alpha\eta\tau_{max}}2\|\nabla F(\boldsymbol{w}_t)\|^2+\frac{\alpha\eta L^{2} \tau_{max} \Delta_t^2}2
	\end{equation}
	
	After the transformation we get
	
	$$\|\nabla F(\boldsymbol{w}_t)\|^2\leq\frac{2\big(F(\boldsymbol{w}_t)-F(\boldsymbol{w}_{t+1})\big)}{\alpha\eta\tau_{max}}+ L ^{2}\Delta_t^2$$
	
	Finally, summing from $t = 1$ to $T$, we obtain
	
	$$\frac1T\sum_{t=1}^T\|\nabla F(\boldsymbol{w}_t)\|^2\leq\frac{2\left(F(\boldsymbol{w}_1)-F(\boldsymbol{w}^*)\right)}{\alpha\eta\tau_{max} T}+\frac{ L ^2}T\sum_{t=1}^T\Delta_t^2$$
	
	That completes the proof.
\end{proof}

\begin{lemma}
	Under Assumptions \ref{L-smooth}, \ref{model diff} and \ref{bound-error}, if the learning rate $\eta$ fulfills $\eta\leq\frac{\sqrt{3}\beta}{6 L\tau_{max}}$ with $t\in[1,T-1]$, then the following inequalities hold:
	$$\Delta_{t+1}^{2}\leq\frac12{\beta^{2}}\eta^{2}\tau_{\max}^{2}\varsigma^{2}+\frac12{\beta^{2}}\eta^{2}\tau_{\max}^{2}\|\nabla F(\boldsymbol{w}_{t+1})\|^{2}$$
    and,
	$$\Delta_1^2\leq\frac6{\beta^2}\eta^2\tau_{\mathrm{max}}^2$$
    and finally,
	$$\sum_{t=1}^T\Delta_t^2\leq\frac{12}{\beta^{2}}(T-1)\eta^{2}\tau_{max}^{2}\left(\varsigma^{2}+\sum_{t=1}^{T-1}\parallel\nabla F(\boldsymbol{w}_{t})\parallel^{2}\right)+\frac{6}{\beta^{2}}\eta^{2}\tau_{max}^{2}\varsigma^{2}$$
	
\end{lemma}

\begin{proof}
    Without the loss of generality, for all the all the $t\in[2,T]$,
	
    $$\begin{aligned}
        \boldsymbol{w}_{t+1}& =\sum_{i=1}^{N}p_{i}\boldsymbol{w}_{(t,i)}^{\tau_{i}+1}  \\
        &=\sum_{i=1}^{N}\left.p_{i}\left(\boldsymbol{w}_{(t,i)}^{\lambda=1}-\eta\sum_{\lambda=1}^{\tau_{i}}\nabla F_{i}(\boldsymbol{w}_{(t,i)}^{\lambda};\boldsymbol{x}_{\sigma_{i}})\right)\right.  \\
        &=\sum_{i=1}^{N}p_{i}\boldsymbol{w}_{(t,i)}^{\lambda=1}-\sum_{i=1}^{N}p_{i}\eta\sum_{\lambda=1}^{\tau_{i}}\nabla F_{i}(\boldsymbol{w}_{(t,i)}^{\lambda};\boldsymbol{x}_{\sigma_{i}}) \\
        &=\sum_{i=1}^Np_i\boldsymbol{w}_t-\eta\sum_{i=1}^N\sum_{\lambda=1}^{\boldsymbol{\tau}_i}p_i\nabla F_i(\boldsymbol{w}_{(t,i)}^\lambda;\boldsymbol{x}_{\sigma_i}) \\
        &=\boldsymbol{w}_{t}-\eta\sum_{i=1}^{N}p_{i}\tau_{i}\sum_{\lambda=1}^{\tau_{i}}\frac1{\tau_{i}}\nabla F_{i}(\boldsymbol{w}_{t},\boldsymbol{x}_{\sigma_{i}})-\eta\sum_{i=1}^{N}p_{i}\tau_{i}\sum_{\lambda=1}^{\tau_{i}}\frac1{\tau_{i}}\Big(\nabla F_{i}(\boldsymbol{w}_{(t,i)}^{\lambda};\boldsymbol{x}_{\sigma_{i}})-\nabla F_{i}(\boldsymbol{w}_{t};\boldsymbol{x}_{\sigma_{i}})\Big)
    \end{aligned}$$

    Now add and subtract

    $$\eta\sum_{i=1}^{N}p_{i}\tau_{i}\sum_{\lambda=1}^{\tau_{i}}\frac{1}{\tau_{i}}\nabla F_{i}(\boldsymbol{w}_{t})=\eta\sum_{i=1}^{N}p_{i}\tau_{i}\nabla F_{i}(\boldsymbol{w}_{t})$$
    which gives	
    $$\begin{aligned}
        \boldsymbol{w}_{t+1}& =\boldsymbol{w}_{t}-\left(\eta\sum_{i=1}^{N}p_{i}\tau_{i}\sum_{\lambda=1}^{\pi_{i}}\frac{1}{\tau_{i}}\nabla F_{i}(\boldsymbol{w}_{t};\boldsymbol{x}_{\sigma_{i}})-\eta\sum_{i=1}^{N}p_{i}\tau_{i}\sum_{\lambda=1}^{\pi_{i}}\frac{1}{\tau_{i}}\nabla F_{i}(\boldsymbol{w}_{t})\right)-\eta\sum_{i=1}^{N}p_{i}\tau_{i}\nabla F_{i}(\boldsymbol{w}_{t})  \\
        &-\eta\sum_{i=1}^{N}p_{i}\tau_{i}\sum_{\lambda=1}^{\tau_{i}}\frac{1}{\tau_{i}}\big(\nabla F_{i}\big(\boldsymbol{w}_{(t,i)}^{\lambda};x_{\sigma_{i}}\big)-\nabla F_{i}\big(\boldsymbol{w}_{t};x_{\sigma_{i}}\big)\big)\big| \\
        &=\boldsymbol{w}_{t}-\eta\sum_{i=1}^{N}p_{i}\tau_{i}\sum_{\lambda=1}^{\boldsymbol{\tau}_{i}}\frac{1}{\tau_{i}}\Big(\nabla F_{i}\big(\boldsymbol{w}_{t};\boldsymbol{x}_{\sigma_{i}}\big)-\nabla F_{i}(\boldsymbol{w}_{t})\Big)-\eta\sum_{i=1}^{N}p_{i}\boldsymbol{\tau}_{i}\nabla F_{i}(\boldsymbol{w}_{t}) \\
        &-\eta\sum_{i=1}^Np_i\tau_i\sum_{\lambda=1}^{\tau_i}\frac1{\tau_i}\bigl(\nabla F_i(\boldsymbol{w}_{(t,i)}^\lambda;x_{\sigma_i})-\nabla F_i(\boldsymbol{w}_t;x_{\sigma_i})\bigr)
    \end{aligned}$$

    We further add and subtract
	
    $$\eta\sum_{i=1}^{N}p_{i}\tau_{i}\nabla F(\boldsymbol{w}_{t-1})$$
    to arrive at
	
    $$\begin{aligned}
        \boldsymbol{w}_{t+1}& =\boldsymbol{w}_t-\eta\sum_{i=1}^Np_i\tau_i\sum_{\lambda=1}^{\boldsymbol{\tau}_i}\frac1{\tau_i}\Big(\nabla F_i\big(\boldsymbol{w}_t;\boldsymbol{x}_{\sigma_i}\big)-\nabla F_i(\boldsymbol{w}_t)\Big)-\eta\sum_{i=1}^Np_i\tau_i\nabla F_i(\boldsymbol{w}_t)  \\
        &+\eta\sum_{i=1}^{N}p_{i}\tau_{i}\nabla F(\boldsymbol{w}_{t-1})-\eta\sum_{i=1}^{N}p_{i}\tau_{i}\nabla F(\boldsymbol{w}_{t-1}) -\eta\sum_{i=1}^{N}p_{i}\tau_{i}\sum_{\lambda=1}^{\tau_{i}}\frac1{\tau_{i}}\Big(\nabla F_{i}\big(\boldsymbol{w}_{(t,i)}^{\lambda};x_{\sigma_{i}}\big)-\nabla F_{i}\big(\boldsymbol{w}_{t};x_{\sigma_{i}}\big)\Big) \\
        &=\boldsymbol{w}_{t}-\eta\sum_{i=1}^{N}p_{i}\tau_{i}\sum_{\lambda=1}^{\tau_{i}}\frac1{\tau_{i}}\Big(\nabla F_{i}(\boldsymbol{w}_{t};\boldsymbol{x}_{\sigma_{i}})-\nabla F_{i}(\boldsymbol{w}_{t})\Big)-\eta\sum_{i=1}^{N}p_{i}\tau_{i}\big(\nabla F_{i}(\boldsymbol{w}_{t})-\nabla F(\boldsymbol{w}_{t-1})\big) \\
        &-\eta\sum_{i=1}^Np_i\tau_i\nabla F(\boldsymbol{w}_{t-1})-\eta\sum_{i=1}^Np_i\tau_i\sum_{\lambda=1}^{\tau_i}\frac1{\tau_i}\Big(\nabla F_i\big(\boldsymbol{w}_{(t,i)}^\lambda;\boldsymbol{x}_{\sigma_i}\big)-\nabla F_i\big(\boldsymbol{w}_t;\boldsymbol{x}_{\sigma_i}\big)\Big)
    \end{aligned}$$

    We can now re-arrange, take norms on both sides and apply the Cauchy-Schwarz inequality
	
    $$\begin{aligned}
        \|\boldsymbol{w}_{t+1}-\boldsymbol{w}_{t}\|^{2}& \leq4\left\|\eta\sum_{i=1}^Np_i\tau_i\sum_{\lambda=1}^{\tau_i}\frac{1}{\tau_i}\big(\nabla F_i\big(\boldsymbol{w}_t;x_{\sigma_i}\big)-\nabla F_i(\boldsymbol{w}_t)\big)\right\|^2  \\
        &+4\left\|\eta\sum_{i=1}^{N}p_{i}\tau_{i}\big(\nabla F_{i}(\boldsymbol{w}_{t})-\nabla F(\boldsymbol{w}_{t-1})\big)\right\|^{2}+4\left\|\eta\sum_{i=1}^{N}p_{i}\tau_{i}\nabla F(\boldsymbol{w}_{t-1})\right\|^{2} \\
        &+4\left\|\eta\sum_{i=1}^{N}p_{i}\tau_{i}\sum_{\lambda=1}^{\tau_{i}}\frac{1}{\tau_{i}}\Big(\nabla F_{i}\big(\boldsymbol{w}_{(t,i)}^{\lambda};x_{\sigma_{i}}\big)-\nabla F_{i}\big(\boldsymbol{w}_{t};x_{\sigma_{i}}\big)\Big)\right\|^{2}
    \end{aligned}$$

    There are four different paradigm terms on the right hand side. For the first term, we apply the Assumption \ref{bound-error} and the Jensen inequality	

    $$\begin{aligned}
        &\left\|\eta\sum_{i=1}^Np_i\tau_i\sum_{\lambda=1}^{\tau_i}\frac1{\tau_i}\left(\nabla F_i(\boldsymbol{w}_t;x_{\sigma_i})-\nabla F_i(\boldsymbol{w}_t)\right)\right\|^2 \\
        &\leq\eta^2\sum_{i=1}^Np_i\left\|\tau_i\sum_{\lambda=1}^{\tau_i}\frac1{\tau_i}\big(\nabla F_i\big(\boldsymbol{w}_t;x_{\sigma_i}\big)-\nabla F_i(\boldsymbol{w}_t)\big)\right\|^2 \\
        &\leq\eta^2\sum_{i=1}^Np_i\tau_i^2\sum_{\lambda=1}^{\tau_i}\frac1{\tau_i}\left\|\left(\nabla F_i(\boldsymbol{w}_t;x_{\sigma_i})-\nabla F_i(\boldsymbol{w}_t)\right)\right\|^2 \\
        &=\eta^2\sum_{i=1}^Np_i\tau_i^2\frac1{\tau_i}\sum_{\lambda=1}^{\tau_i}\left\|\left(\nabla F_i(\boldsymbol{w}_t;x_{\sigma_i})-\nabla F_i(\boldsymbol{w}_t)\right)\right\|^2 \\
        &\leq\eta^2\tau_{max}^2\sum_{i=1}^Np_i\frac1{\tau_i}\sum_{\lambda=1}^{\tau_i}\left\|\left(\nabla F_i(\boldsymbol{w}_t;x_{\sigma_i})-\nabla F_i(\boldsymbol{w}_t)\right)\right\|^2 \\
        &\leq\eta^2\tau_{max}^2\sum_{i=1}^Np_i\varsigma^2 \\
        &=\eta^2\tau_{max}^2\varsigma^2
    \end{aligned}$$

    For the second term, we apply the Assumption \ref{L-smooth}, \ref{model diff}, Proposition \ref{proposition} and the Jensen inequality
	
    $$\begin{aligned}
        &\left\|\eta\sum_{i=1}^Np_i\tau_i(\nabla F_i(\boldsymbol{w}_t)-\nabla F(\boldsymbol{w}_{t-1}))\right\|^2 \\
        &\leq\eta^2\sum_{i=1}^Np_i\tau_i^2\|\nabla F_i(\boldsymbol{w}_t)-\nabla F(\boldsymbol{w}_{t-1})\|^2 \\
        &\leq\eta^{2}\sum_{i=1}^{N}p_{i}\tau_{i}^{2}L^{2}\|\boldsymbol{w}_{t-1}-\boldsymbol{w}_{t}\|^{2} \\
        &=\eta^2\sum_{i=1}^Np_i\tau_i^2L^2\left\|\boldsymbol{w}_{t-1}-\sum_{i=1}^Np_i\boldsymbol{w}_{(t-1,i)}^{\tau_i+1}\right\|^2 \\
        &=\eta^{2}L^{2}\sum_{i=1}^{N}p_{i}\tau_{i}^{2}\left\|\sum_{i=1}^{N}p_{i}(\boldsymbol{w}_{t-1}-\boldsymbol{w}_{(t-1,i)}^{\tau_{i}+1})\right\|^{2} \\
        &\leq\eta^2L^2\tau_{max}^2\sum_{i=1}^Np_i\Delta_{t-1}^2 \\
        &=\eta^{2}L^{2}\tau_{max}^{2}\Delta_{t-1}^{2}
    \end{aligned}$$

    For the third term, we apply the Jensen inequality

    $$\begin{aligned}
        &\left.\left\|\eta\sum_{i=1}^Np_i\right.\tau_i\nabla F(\boldsymbol{w}_{t-1})\right\|^2 \\
        &\leq\eta^2\sum_{i=1}^Np_i\tau_i^2\|\nabla F(\boldsymbol{w}_{t-1})\|^2 \\
        &\leq\eta^{2}\tau_{max}^{2}\sum_{i=1}^{N}p_{i}\|\nabla F(\boldsymbol{w}_{t-1})\|^{2} \\
        &=\eta^{2}\tau_{max}^{2}\|\nabla F(\boldsymbol{w}_{t-1})\|^{2}
    \end{aligned}$$

        For the last term, we apply the Assumption \ref{L-smooth}, \ref{model diff} and the Jensen inequality
	
    $$\begin{aligned}
        &\left\|\eta\sum_{i=1}^Np_i\tau_i\sum_{\lambda=1}^{\tau_i}\frac1{\tau_i}\Big(\nabla F_i(\boldsymbol{w}_{(t,i)}^\lambda;\boldsymbol{x}_{\sigma_i})-\nabla F_i(\boldsymbol{w}_t;\boldsymbol{x}_{\sigma_i})\Big)\right\|^2 \\
        &\leq\eta^2\sum_{i=1}^Np_i\tau_i^2\sum_{\lambda=1}^{\tau_i}\frac1{\tau_i}\|\nabla F_i(\boldsymbol{w}_{(t,i)}^\lambda;x_{\sigma_i})-\nabla F_i(\boldsymbol{w}_t;x_{\sigma_i})\|^2 \\
        &\leq\eta^{2}\sum_{i=1}^{N}p_{i}\tau_{i}^{2}\sum_{\lambda=1}^{\tau_{i}}\frac{1}{\tau_{i}}L^{2}\|\boldsymbol{w}_{(t,i)}^{\lambda}-\boldsymbol{w}_{t}\|^{2} \\
        &=\eta^{2}L^{2}\sum_{i=1}^{N}p_{i}{\tau_{i}}^{2}\Delta_{t}^{2} \\
        &\leq\eta^{2}L^{2}\tau_{max}^{2}\Delta_{t}^{2}
    \end{aligned}$$

        Aggregate these terms and apply Assumption \ref{model diff}

    $$\begin{aligned}
        \beta^2\Delta_t^2 &\leq\|\boldsymbol{w}_{t+1}-\boldsymbol{w}_{t}\|^{2} \\
        &\leq4\eta^{2}\tau_{max}^{2}\varsigma^{2}+4\eta^{2}L^{2}\tau_{max}^{2}\Delta_{t-1}^{2}+4\eta^{2}L^{2}\tau_{max}^{2}\Delta_{t}^{2}+4\eta^{2}\tau_{max}^{2}\|\nabla F(\boldsymbol{w}_{t-1})\|^{2}
    \end{aligned}$$

    When $t=t+1$, $t\in[1,T-1]$
    $$(\beta^{2}-4\eta^{2}L^{2}\tau_{max}^{2})\Delta_{t+1}^{2}\leq4\eta^{2}\tau_{max}^{2}\varsigma^{2}+4\eta^{2}L^{2}\tau_{max}^{2}\Delta_{t}^{2}+4\eta^{2}\tau_{max}^{2}\|\nabla F(\boldsymbol{w}_{t})\|^{2}$$
		
    If $\eta\leq\frac{\sqrt{3}\beta}{6 L\tau_{max}}$, so we have $\eta^{2}L^{2}\tau_{max}^{2}\leq\frac{\beta^{2}}{12}$ and $\beta^{2}-4\eta^{2}L^{2}\tau_{max}^{2}\geq\frac{2\beta^{2}}{3}$, also this condition is consistent with $\alpha\eta L \tau_{max}\leq1$ in Lemma \ref{converge-bound}, then
	
	$$\begin{aligned}
        \frac{2\beta^{2}}{3}\Delta_{t+1}^{2}& \leq\frac{\beta^{2}}{3}\Delta_{t}^{2}+4\eta^{2}\tau_{max}^{2}\varsigma^{2}+4\eta^{2}\tau_{max}^{2}\|\nabla F(\boldsymbol{w}_{t+1})\|^{2}  \\
        \Delta_{t+1}^{2}& \leq\frac{1}{2}\Delta_{t}^{2}+\frac{3}{2\beta^{2}}\left(4\eta^{2}\tau_{max}^{2}\varsigma^{2}+4\eta^{2}\tau_{max}^{2}\|\nabla F(\boldsymbol{w}_{t+1})\|^{2}\right)  \\
        &=\left(\frac{1}{2}\right)^{t}\Delta_{1}^{2}+\sum_{j=1}^{t}\left.\left(\frac{1}{2}\right)^{j-1}\left(\frac{6}{\beta^{2}}\eta^{2}\tau_{max}^{2}\varsigma^{2}+\frac{6}{2\beta^{2}}\eta^{2}\tau_{max}^{2}\|\nabla F(\boldsymbol{w}_{t+1})\|^{2}\right)\right. 
	\end{aligned}$$
	
	The right-hand side of the inequality increases with $t$. When $t\to\infty $, then
	
	$$\Delta_{t+1}^{2}\leq\frac{12}{\beta^{2}}\eta^{2}\tau_{max}^{2}\varsigma^{2}+\frac{12}{\beta^{2}}\eta^{2}\tau_{max}^{2}\|\nabla F(\boldsymbol{w}_{t+1})\|^{2} $$
	
	When $t=0$, we can apply the similar trick and get
	
	$$\begin{aligned}
		\Delta_{1}^{2}&\leq\frac{1}{2}\Delta_{0}^{2}+\frac{3}{2\beta^{2}}\left(4\eta^{2}\tau_{max}^{2}\varsigma^{2}+4\eta^{2}\tau_{max}^{2}\|\nabla F(\boldsymbol{w}_{0})\|^{2}\right) \\
        &=\frac{6}{\beta^{2}}\eta^{2}\tau_{max}^{2}\varsigma^{2} 
	\end{aligned}$$
	
	Summing from $t = 1$ to $T$, we will get
	
	$$\begin{aligned}
		\sum_{t=1}^{T}\Delta_{t}^{2}& =\left(\sum_{t=1}^{T-1}\Delta_{t+1}^{2}\right)+\Delta_{1}^{2}  \\
        &\leq\sum_{t=1}^{T-1}\left.(\frac{12}{\beta^{2}}\eta^{2}\tau_{max}^{2}\varsigma^{2}+\frac{12}{\beta^{2}}\eta^{2}\tau_{max}^{2}(T-1)\|\nabla F(\boldsymbol{w}_{t+1})\|^{2}\right)+\frac{6}{\beta^{2}}\eta^{2}\tau_{max}^{2}\varsigma^{2} \\
        &=\frac{12}{\beta^{2}}(T-1)\eta^{2}\tau_{max}^{2}\left(\varsigma^{2}+\sum_{t=1}^{T-1}\parallel\nabla F(\boldsymbol{w}_{t})\parallel^{2}\right)+\frac{6}{\beta^{2}}\eta^{2}\tau_{max}^{2}\varsigma^{2}
	\end{aligned}$$
	
	That completes the proof.
	
\end{proof}

\begin{theorem}
	In the BHerd strategy, under Assumptions \ref{L-smooth}, \ref{model diff} and \ref{bound-error}, if the learning rate $\eta$ fulfills $\eta\leq\frac{\sqrt{3}\beta}{6 L\tau_{max}}$ and the total global round $T$ fulfills $ T\leq\frac{\beta^{2}}{24\eta^{2}L^{2}\tau_{max}^{2}} $, then it converges at the rate
	
	$$\frac{1}{T}\sum_{t=1}^{T}\parallel\nabla F(\boldsymbol{w}_{t})\parallel^{2}\leq\frac{4(F(\boldsymbol{w}_{1})-F(\boldsymbol{w}^{*}))}{\alpha\eta\tau_{max}T}+\frac{\varsigma^{2}}{T}+\frac{\varsigma^{2}}{2T^{2}}$$
	
\end{theorem}

\begin{proof}
	Substituting the conclusion of Lemma \ref{delta-bound} into Lemma \ref{converge-bound}, we get
	
	$$
	\begin{aligned}
		&\frac{1}{T-1}\sum_{t=1}^{T-1}\|\nabla F(\boldsymbol{w}_t)\|^2 \\
		&\leq\frac{1}{T-1}\sum_{t=1}^{T}\|\nabla F(\boldsymbol{w}_t)\|^2 \\
		&\leq\frac{2(F(\boldsymbol{w}_1)-F(\boldsymbol{w}^*))}{a\eta\boldsymbol{v}_{max}(T-1)}+\frac{L^2}{(T-1)}\sum_{t=1}^T\Delta_t^2\\
        &\leq\frac{2(F(\boldsymbol{w}_1)-F(\boldsymbol{w}^*))}{a\eta\boldsymbol{\tau}_{max}(T-1)}+\frac{L^2}{(T-1)}\bigg(\frac{12}{\beta^2}(T-1)\eta^2\tau_{max}^2\bigg(\varsigma^2+\sum_{t=1}^{T-1}\parallel\nabla F(\boldsymbol{w}_t)\parallel^2\bigg)+\frac{6}{\beta^2}\eta^2\tau_{max}^2\varsigma^2\bigg)
	\end{aligned} 
	$$
	
	Re-arrange, we get
	
	$$\frac{1-\frac{12}{\beta^2}(T-1)\eta^2L^2\tau_{max}^2}{T-1}\sum_{t=1}^{T-1}\parallel\nabla F(\boldsymbol{w}_t)\parallel^2\leq\frac{2\left(F(\boldsymbol{w}_1)-F(\boldsymbol{w}^*)\right)}{\alpha\eta\tau_{max}(T-1)}+\frac{12}{\beta^2}\eta^2L^2\tau_{max}^2\varsigma^2+\frac{6\eta^2L^2\tau_{max}^2\varsigma^2}{\beta^2(T-1)}$$
	
	If $(T-1)\leq\frac{\beta^{2}}{24\eta^{2}L^{2}\tau_{max}^{2}}$, we have $\frac{12}{\beta^2}(T-1)\eta^2L^2\tau_{max}^2\leq\frac12$, thus
	
	$$
	\begin{aligned}
		\frac{1}{T-1}\sum_{t=1}^{T-1}\parallel\nabla F(\boldsymbol{w}_{t})\parallel^{2}& \leq\frac{4(F(\boldsymbol{w}_{1})-F(\boldsymbol{w}^{*}))}{\alpha\eta\tau_{max}(T-1)}+\frac{24}{\beta^{2}}\eta^{2}L^{2}\tau_{max}^{2}\varsigma^{2}+\frac{12\eta^{2}L^{2}\tau_{max}^{2}\varsigma^{2}}{\beta^{2}(T-1)}  \\
        &\leq\frac{4(F(\boldsymbol{w}_{1})-F(\boldsymbol{w}^{*}))}{\alpha\eta\tau_{max}(T-1)}+\frac{\varsigma^{2}}{(T-1)}+\frac{\varsigma^{2}}{2(T-1)^{2}}
	\end{aligned}
	$$
	
	Replace $T-1$ with $T$, we get $t\in[2,T]$, $ T\leq\frac{\beta^{2}}{24\eta^{2}L^{2}\tau_{max}^{2}} $ and
	
	$$\frac{1}{T}\sum_{t=1}^{T}\parallel\nabla F(\boldsymbol{w}_t)\parallel^2 \leq\frac{4\big(F(\boldsymbol{w}_1)-F(\boldsymbol{w}^*)\big)}{\alpha\eta\tau_{max} T}+\frac{\varsigma^2}{T}+\frac{\varsigma^2}{8 T^{2}} $$
	
	That completes the proof.
	
\end{proof}

\begin{theorem}
	Under Assumption \ref{L-smooth}, \ref{model diff}, \ref{bound-error} and \ref{PL} (PL condition), if $\eta\leq\frac{\sqrt{3}\beta}{6 L\tau_{max}}$ and $ T\leq\frac{\beta^{2}}{24\alpha\eta^{3}L^{3}\tau_{max}^{3}} $ hold, then the BHerd strategy converges at the upper bound
		$$F(\boldsymbol{w}_T)-F(\boldsymbol{w}^*)=\mathcal{O}(\frac{\varsigma^{2}}{T\eta L\tau_{max}\gamma})+\mathcal{O}(\frac{F}{T\eta L\tau_{max}})+\mathcal{O}(\frac{\alpha\varsigma^{2}}{TL}), $$
		where $\mathcal{O}$ swallows all other constants and $F:=F(\boldsymbol{w}_1)-F(\boldsymbol{w}^*)$.
\end{theorem}
	
\begin{proof}
		Since the this theorem is a special case of Theorem \ref{theo-bound}, we can just borrow the Equation \eqref{inter-formula} there and get for all the $t=1,...,T-1$
		
		$$\begin{aligned}
		    F(\boldsymbol{w}_{t+1})& \leq F(\boldsymbol{w}_{t})+\frac{\alpha\eta L^{2}\tau_{max}\Delta_{t}^{2}}{2}-\frac{\alpha\eta\tau_{max}}{2}\parallel\nabla F(\boldsymbol{w}_{t})\parallel^{2}  \\
            &\leq F(\boldsymbol{w}_{t})+\frac{\eta L^{2}\tau_{max}\Delta_{t}^{2}}{2}-\alpha\eta\tau_{max}\gamma(F(\boldsymbol{w}_{t})-F(\boldsymbol{w}^{*})) 
		\end{aligned}$$
		Defined $\rho=1-\alpha\eta\tau_{max}\gamma$ and adding $F(\boldsymbol{w}_t)-F(\boldsymbol{w}^*)$ to both sides gives
		$$\begin{aligned}
			F(\boldsymbol{w}_t)-F(\boldsymbol{w}^*)+F(\boldsymbol{w}_{t+1})& \leq F(\boldsymbol{w}_t)+\frac{\alpha\eta L^{2} \tau_{max}\Delta_t^2}{2}+(1-\alpha\eta\tau_{max}\gamma)\big(F(\boldsymbol{w}_t)-F(\boldsymbol{w}^*)\big)  \\
			&=F(\boldsymbol{w}_t)+\frac{\alpha\eta L^{2} \tau_{max}\Delta_t^2}2+\rho\big(F(\boldsymbol{w}_t)-F(\boldsymbol{w}^*)\big)\\
			F(\boldsymbol{w}_{t+1})-F(\boldsymbol{w}^*)&\overset{\text{Rearrange }}{\leq}\rho\big(F(\boldsymbol{w}_t)-F(\boldsymbol{w}^*)\big)+\frac{\alpha\eta L^{2} \tau_{max}\Delta_t^2}2
		\end{aligned}$$
		
		Recursively apply it from $t = 1$ to $T-1$, we obtain
		$$\begin{aligned}
			F(\boldsymbol{w}_T)-F(\boldsymbol{w}^*) &\leq\rho^T\big(F(\boldsymbol{w}_1)-F(\boldsymbol{w}^*)\big)+\frac{\alpha\eta L^{2} \tau_{max}}2\sum_{t=1}^{T-1}\rho^{T-1-t}\Delta_t^2 \\
			&=\rho^TF+\frac{\alpha\eta L^{2} \tau_{max}}2\sum_{t=1}^{T-1}\rho^{T-1-t}\Delta_t^2 \\
			&\overset{T\to\infty}{=}\frac{\alpha\eta L^{2} \tau_{max}}2\sum_{t=1}^{T-1}\rho^{T-1-t}\Delta_t^2 
		\end{aligned}$$
		
		In Lemma \ref{delta-bound}, we have $\Delta_{t+1}^2\leq\frac{12}{\beta^{2}}\eta^{2}\tau_{max}^{2}\varsigma^{2}+\frac{12}{\beta^{2}}\eta^{2}\tau_{max}^{2}\|\nabla F(\boldsymbol{w}_{t+1})\|^{2}$ and $\Delta_1^2\leq\frac{6}{\beta^{2}}\eta^{2}\tau_{max}^{2}\varsigma^{2} $, thus
		
		$$\begin{aligned}
            \sum_{t=1}^{T-1}\rho^{T-1-t}\Delta_{t}^{2}& \leq\rho^{T-2}\Delta_{1}^{2}+\sum_{t=1}^{T-2}\rho^{T-2-t}\Delta_{t+1}^{2}  \\
            &\leq\frac{6}{\beta^{2}}\rho^{T-2}\eta^{2}\tau_{max}^{2}\varsigma^{2}+\sum_{t=1}^{T-2}\left.\rho^{T-2-t}\left(\frac{12}{\beta^{2}}\eta^{2}\tau_{max}^{2}\varsigma^{2}+\frac{12}{\beta^{2}}\eta^{2}\tau_{max}^{2}\|\nabla F(\boldsymbol{w}_{t+1})\|^{2}\right)\right.  \\
            &=\frac{6}{\beta^{2}}\rho^{T-2}\eta^{2}\tau_{max}^{2}\varsigma^{2}+\frac{12\eta^{2}\tau_{max}^{2}\varsigma^{2}(1-\rho^{T-2})}{\beta^{2}(1-\rho)}+\frac{12}{\beta^{2}}\eta^{2}\tau_{max}^{2}\sum_{t=1}^{T-2}\rho^{T-2-t}\parallel\nabla F(\boldsymbol{w}_{t})\parallel^{2} \\
            &\leq\frac{6}{\beta^{2}}\rho^{T-2}\eta^{2}\tau_{max}^{2}\varsigma^{2}+\frac{12\eta^{2}\tau_{max}^{2}\zeta^{2}(1-\rho^{T-2})}{\beta^{2}(1-\rho)}+\frac{12}{\beta^{2}}\eta^{2}\tau_{max}^{2}\sum_{t=1}^{T-1}\parallel\nabla F(\boldsymbol{w}_{t})\parallel^{2} \\
            &\leq\frac{6}{\beta^{2}}\rho^{T-2}\eta^{2}\tau_{max}^{2}\varsigma^{2}+\frac{12\eta^{2}\tau_{max}^{2}\zeta^{2}(1-\rho^{T-2})}{\beta^{2}(1-\rho)}+\frac{12}{\beta^{2}}\eta^{2}\tau_{max}^{2}\left(\frac{4F}{\alpha\eta\tau_{max}}+\varsigma^{2}+\frac{\varsigma^{2}}{2(T-1)}\right) \\
            &\overset{T\to\infty}{=}\frac{12\eta^{2}\tau_{max}^{2}\zeta^{2}}{\beta^{2}\alpha\eta\tau_{max}\gamma}+\frac{12}{\beta^{2}}\eta^{2}\tau_{max}^{2}\left(\frac{4F}{\alpha\eta\tau_{max}}+\varsigma^{2}\right) \\
            &=\frac{12\eta\tau_{max}\varsigma^{2}}{\alpha\beta^{2}\gamma}+\frac{48\eta\tau_{max}F}{\alpha\beta^{2}}+\frac{12\eta^{2}\tau_{max}^{2}\varsigma^{2}}{\beta^{2}}
        \end{aligned}$$
		
		Put it back to the previous inequality yields
		
		$$\begin{aligned}
            &F(\boldsymbol{w}_{T})-F(\boldsymbol{w}^{*})\leq\frac{\alpha\eta L^{2}\tau_{max}}{2}\sum_{t=1}^{T-1}\rho^{T-1-t}\Delta_{t}^{2} \\
            &\leq\frac{\alpha\eta L^{2}\tau_{max}}{2}\left(\frac{12\eta\tau_{max}\varsigma^{2}}{\alpha\beta^{2}\gamma}+\frac{48\eta\tau_{max}F}{\alpha\beta^{2}}+\frac{12\eta^{2}\tau_{max}^{2}\varsigma^{2}}{\beta^{2}}\right) \\
            &=\frac{\alpha\eta^{2}L^{2}\tau_{max}^{2}}{2}\left(\frac{12\varsigma^{2}}{\alpha\beta^{2}\gamma}+\frac{48F}{\alpha\beta^{2}}+\frac{12\eta\tau_{max}\varsigma^{2}}{\beta^{2}}\right) \\
            &\leq\frac{\beta^{2}}{48T\eta L\tau_{max}}\left(\frac{12\varsigma^{2}}{\alpha\beta^{2}\gamma}+\frac{48F}{\alpha\beta^{2}}+\frac{12\eta\tau_{max}\varsigma^{2}}{\beta^{2}}\right) \\
            &=\frac{\varsigma^{2}}{4T\alpha\eta L\tau_{max}\gamma}+\frac{F}{T\alpha \eta L\tau_{max}}+\frac{\varsigma^{2}}{4TL}\\
            &=\mathcal{O}(\frac{\varsigma^{2}}{T\eta L\tau_{max}\gamma})+\mathcal{O}(\frac{F}{T\eta L\tau_{max}})+\mathcal{O}(\frac{\alpha\varsigma^{2}}{TL})
        \end{aligned}$$
		
		The third term is due to the fact that $ T\leq\frac{\beta^{2}}{24\alpha\eta^{3}L^{3}\tau_{max}^{3}}$ yields $\frac{\alpha\eta^{2}L^{2}\tau_{max}^{2}}{2}\leq\frac{\beta^{2}}{48T\eta L\tau_{max}}$. It's worth noting that condition $\alpha\eta L \tau_{max}\leq1$ follows from condition $\eta\leq\frac{\sqrt{3}\beta}{6 L\tau_{max}}$ in Lemma \ref{delta-bound}, so that when $ T\leq\frac1{24\alpha\eta^3 L^3\tau_{max}^3}$, it also satisfies $ T\leq\frac{\beta^{2}}{24\eta^{2}L^{2}\tau_{max}^{2}} $ in Theorem \ref{theo-bound}.
		
		That completes the proof.
\end{proof}

\section{Pseudo Code of BHerd-FedAvg and Baseline Algorithms}
\label{alg-result}
Following the same parameters $T, B, E, N$ and $\boldsymbol{w}_{1}$, we give  the pseudo-codes for the FedAvg (Algorithm \ref{fedavg-alg}), BHerd-FedAvg (Algorithm \ref{fed-herding}) and GraB-FedAvg (Algorithm \ref{fedgrab-alg}) algorithms here.
\begin{figure}[!t]
	\begin{algorithm}[H]
		\caption{General Framework of FedAvg}
		\label{fedavg-alg}
		\textbf{Input}: Total global round $T$, local iterations $\tau_i$, initialized weight $\boldsymbol{w}_{1}$, learning rate $\eta$.
		
		\begin{algorithmic}[1] 
			\For{$t=1,...,T$}
			\For{$i=1,...,N$}
			\For{$\lambda = 1,...,\tau_i$}
			\State Compute the local gradient: $ \nabla F_i(\boldsymbol{w}_{(t,i)}^\lambda;\boldsymbol{x}_{\sigma_i}) $ using Equation \eqref{local-update}.
			\EndFor
			\State Sum all local gradients: $\sum_{\lambda=1}^{\tau_i}\nabla F_i(\boldsymbol{w}_{(t,i)}^\lambda;\boldsymbol{x}_{\sigma_i})$.
			\EndFor
			\State Update the global model $\boldsymbol{w}_{t+1}$ using Equation \eqref{fedavg}.
			\EndFor
			\State \textbf{return} $\boldsymbol{w}_{T+1}$
		\end{algorithmic}
	\end{algorithm}
\end{figure}

\begin{figure}[!t]
	\begin{algorithm}[H]
		\caption{Herding with Greedy Ordering}
		\label{herding-order}
		\textbf{Input}: A group of gradients $\{\boldsymbol{z}_{\lambda}\}_{\lambda=1}^{\tau}$, the proportion $\alpha$.
		
		\begin{algorithmic}[1] 
			\State Center all the vectors: $\boldsymbol{z}_{\lambda}\leftarrow \boldsymbol{z}_{\lambda}-\frac{1}{\tau}\sum_{\mu=1}^{\tau}\boldsymbol{z}_{\mu}, \forall \lambda\in[\tau]$.
			
			\State Initialize an arbitrary $ \sigma_i^{\pi} $, running partial sum: $\boldsymbol{s}\leftarrow\boldsymbol{0}$, candidate set $\sigma_i \leftarrow \left\lbrace 1,...,\tau\right\rbrace $.
			\For{$\lambda=1,...,\alpha\tau$}
			\State Iterate through $\sigma_i$ and select $\boldsymbol{z}_{\mu}$ from $\sigma_i$ that minimizes $\|\boldsymbol{s}+\boldsymbol{z}_{\mu}\|$.
			\State Remove $\mu$ from $\sigma_i$, update partial sum and order: $\boldsymbol{s}\leftarrow\boldsymbol{s}+\boldsymbol{z}_{\mu}$;  $ \sigma_i^{\pi} \left( \lambda \right) \leftarrow \mu$.
			\EndFor 
			\State \textbf{return} $ \sigma_i^{\pi} $
		\end{algorithmic}
	\end{algorithm}
\end{figure}

\begin{figure}[!t]
	\begin{algorithm}[H]
		\caption{General Framework of BHerd-FedAvg}
		\label{fed-herding}
		\textbf{Input}: Total global round $T$, local iterations $\tau_i$, initialized order $ \sigma_i^{\pi} $ at client $i$, initialized weight $\boldsymbol{w}_{1}$, learning rate $\eta$.
		
		\begin{algorithmic}[1] 
			\For{$t=1,...,T$}
			\For{$i=1,...,N$}
			\For{$\lambda = 1,...,\tau_i$}
			\State Compute the local gradient: $ \nabla F_i(\boldsymbol{w}_{(t,i)}^\lambda;\boldsymbol{x}_{\sigma_i}) $ using Equation \eqref{local-update}.
			\State Store the gradient: $\boldsymbol{z}_{\lambda}\leftarrow \nabla F_i(\boldsymbol{w}_{(t,i)}^\lambda;\boldsymbol{x}_{\sigma_i}) $.
			\EndFor
			\State Generate new order: $ \sigma_i^{\pi} \leftarrow Herding\left( \{\boldsymbol{z}_{\lambda}\}_{\lambda=1}^{\tau}\right) $.
			\State Sum gradients: $\boldsymbol{g}_{(t,\sigma_i^\pi)}$ using Equation \eqref{sum-herd}.
			\EndFor
			\State Update the global model $\boldsymbol{w}_{t+1} $ using Equation \eqref{global-aggregation}.
			\EndFor
			\State \textbf{return} $\boldsymbol{w}_{T+1}$
		\end{algorithmic}
	\end{algorithm}
\end{figure}

\begin{figure}[!ht]
	\begin{algorithm}[H]
		\caption{General Framework of Grab-FedAvg}
		\label{fedgrab-alg}
		\textbf{Input}: Total global round $T$, local iterations $\tau_i$, initialized weight $\boldsymbol{w}_{1}$, learning rate $\eta$.
		
		\begin{algorithmic}[1] 
			\For{$t=1,...,T$}
			\For{$i=1,...,N$}
			\State Initialized the average of local gradient $\boldsymbol{\mu_i} \gets \boldsymbol{0}$, the summation of centered gradient $\boldsymbol{s_i}\gets \boldsymbol{0}$ and the summation of selected gradient $\boldsymbol{g_i}\gets \boldsymbol{0}$.
			\For{$\lambda = 1,...,\tau_i$}
			\State Compute gradient: $ \nabla F_i(\boldsymbol{w}_{(t,i)}^\lambda;\boldsymbol{x}_{\sigma_i}) $ using Equation \eqref{local-update}.
			\State Update: $\boldsymbol{\mu_i} \gets \boldsymbol{\mu_i} + \nabla F_i(\boldsymbol{w}_{(t,i)}^\lambda;\boldsymbol{x}_{\sigma_i})/\tau_i$.
			\State Centralize gradient:  $\boldsymbol{z}_{\lambda}\leftarrow \nabla F_i(\boldsymbol{w}_{(t,i)}^\lambda;\boldsymbol{x}_{\sigma_i}) - \boldsymbol{\mu_i}$
			\If{$\|\boldsymbol{s_i}+\boldsymbol{z}_{\lambda}\|<\|\boldsymbol{s_i}-\boldsymbol{z}_{\lambda}\| $}
			\State $\boldsymbol{s_i}\gets\boldsymbol{s_i}+\boldsymbol{z}_{\lambda}$
			\State $\boldsymbol{g_i}\gets\boldsymbol{g_i}+\nabla F_i(\boldsymbol{w}_{(t,i)}^\lambda;\boldsymbol{x}_{\sigma_i})$
			\Else{}
			\State $\boldsymbol{s_i}\gets\boldsymbol{s_i}-\boldsymbol{z}_{\lambda}$
			\EndIf
			\EndFor
			\State Compute the percentage of gradients selected: $\alpha_i$
			\EndFor
			\State Get the summation of selected gradient $\boldsymbol{g_i}$ at all client.
			
			\State Weighted average: $\alpha_t \gets \sum_{i=1}^Np_i\alpha_i$
			\State $\boldsymbol{w}_{t+1} \leftarrow \boldsymbol{w}_{t}- \frac{\eta}{\alpha_t}\sum_{i=1}^Np_i\boldsymbol{g_i}$ .
			\EndFor
			\State \textbf{return} $\boldsymbol{w}_{T+1}$
		\end{algorithmic}
	\end{algorithm}
\end{figure}

\end{appendices}

\end{document}